\documentclass{article}

\PassOptionsToPackage{numbers,compress}{natbib}
\usepackage{PRIMEarxiv}
\usepackage[utf8]{inputenc} % allow utf-8 input
\usepackage[T1]{fontenc}    % use 8-bit T1 fonts
\usepackage{hyperref}       % hyperlinks
\usepackage{url}            % simple URL typesetting
\usepackage{booktabs}       % professional-quality tables
\usepackage{amsfonts}       % blackboard math symbols
\usepackage{nicefrac}       % compact symbols for 1/2, etc.
\usepackage{microtype}      % microtypography
\usepackage{lipsum}
\usepackage{fancyhdr}       % header
\usepackage{graphicx}       % graphics
\usepackage{authblk}

\usepackage[table]{xcolor}
\usepackage{cite}
\usepackage{adjustbox}
\usepackage{amsmath,amssymb,amsfonts}
\usepackage{algorithmic}
\usepackage{graphicx}
\usepackage{textcomp}
\usepackage{booktabs}
\usepackage[utf8]{inputenc}
\usepackage[T1]{fontenc}
\usepackage{hyperref}
\usepackage{url}
\usepackage{booktabs}
\usepackage{amsfonts}
\usepackage{nicefrac}
\usepackage{microtype}
\usepackage{xcolor}
\usepackage{amssymb}
\usepackage{amsmath}
\usepackage{tcolorbox}
\usepackage{placeins}
\usepackage{multirow}
\usepackage{algorithm}
\usepackage{algorithmic}
\usepackage{amssymb}
\usepackage{amsmath}
\usepackage{amsthm}           % add this line
\usepackage[table]{xcolor}   % ← add this line
\usepackage{booktabs}
\usepackage{wrapfig,lipsum}
\newtheorem{proposition}{Proposition}

\graphicspath{{media/}}     % organize your images and other figures under media/ folder

\title{Solving the Needle-in-a-Haystack Problem in Mammography Vision–Language Model with Differentiable Subset Sampling}

\author{Young Seok Jeon}
\author{Beatrice Brown-Mulry}
\author{Rohan Satya Isaac}
\author{Anjana Dissanayaka}
\author{Frank Li}
\author{Theo Dapamede}
\author{Mohammadreza Chavoshi}
\author{Judy Gichoya}
\author{Hari Trivedi$^{*}$}

\affil{Department of Radiology, Emory University, Atlanta, GA, USA}

\date{}
\begin{document}

\maketitle

\begingroup
\renewcommand{\thefootnote}{*}
\footnotetext{Corresponding Author: Hari Trivedi, <\texttt{hari.trivedi@emory.edu}>}
\endgroup

\begin{abstract}
%Background
There is growing interest in adopting CLIP-style vision--language model (VLM) pretraining for mammography. However, models that directly employ the standard CLIP architecture and training objective exhibit limited zero-shot performance in clinically important tasks such as cancer, finding-type, and BI-RADS predictions.
% hypothesis: 1) low-res, 2) contrastive loss
We argue that this underwhelming performance is due to neglecting two characteristics of mammography data: \textbf{(1)}~its high-res nature, and \textbf{(2)}~homogeneity of radiology reports, largely driven by a predominance of negative/benign findings on examinations.
% method: TopKSigLIP address the two problems
We propose \textbf{TopKSigLIP}, a VLM designed to address these two limitations through a novel architecture and learning objectives.
% method : 1) dypatch
Instead of downscaling high-res mammography images to satisfy GPU memory constraints, TopKSigLIP introduces \textbf{TopK-Patch} module that learns to sample a sparse set of high-res patches likely to contain lesions, sidestepping the resolution--batch size tradeoff of VLM training. The sampled patch locations additionally serve as a built-in localization tool.
% method : 1) sup-sigmoid
To address report homogeneity, we replace the contrastive loss, which falsely repels semantically similar pairs, with a \textbf{Sup-sigmoid} loss. Sup-sigmoid loss extends the sigmoid loss from SigLIP with soft labels derived from structured data.
% result: accurate, interpretable
TopKSigLIP outperforms existing open-source mammography and general medical VLMs on both internal and external benchmarks on density assessment, BI-RADS classification, finding subtyping, and cancer prediction under zero-shot evaluation. TopKSigLIP remains competitive under linear probing despite using a significantly smaller vision encoder and smaller training batches than baselines. The TopK-Patch module additionally achieves superior lesion localization over post-hoc Grad-CAM. Code and weights are made public~\footnote{\texttt{https://github.com/Youngseok0001/TopKSigLIP}}
\end{abstract}

\section{Introduction}
\label{sec:intro}
% why CLIP is good and its sucessful adoption in various medical imaging modalitites.
CLIP-style VLM pretraining~\cite{radford2021learningtransferablevisualmodels,zhai2022litzeroshottransferlockedimage,li2023blip} aligns images and text in a shared embedding space via contrastive learning~\cite{oord2019representationlearningcontrastivepredictive}, yielding transferable representations for various downstream tasks such as image generation~\cite{ramesh2022hierarchical}, visual instruction tuning~\cite{liu2023visualinstructiontuning}, open-vocabulary 
detection~\cite{gu2022openvocabularyobjectdetectionvision}, and segmentation~\cite{kirillov2023segment}.
This paradigm has been widely adopted in medical imaging, including chest radiography~\cite{zhang2022contrastivelearningmedicalvisual,wang2022medclipcontrastivelearningunpaired}, abdominal CT~\cite{Blankemeier_2026}, and pathology~\cite{ding2025multimodal}, as well as recent efforts to unify modalities within a single framework~\cite{codella2024medimageinsightopensourceembeddingmodel,sellergren2025medgemma}. Most works adhere closely to the original CLIP formulation with minimal modification, yet already achieve strong performance on downstream tasks such as medical VQA~\cite{zhang2024development,li2023llava,liu2025gemexlargescalegroundableexplainable} and segmentation~\cite{Liu_2023}.

% existing mammography VLMs' naive CLIP adoption with novelty limited to the multi-view input to vision encoder.
There has been growing interest in training VLMs for digital mammography~\cite{ghosh2024mammoclipvisionlanguagefoundation,chen2024mammoclipleveragingcontrastivelanguageimage,du2025multiviewmultiscalealignmentcontrastive,DuYue_GeometryGuided_MICCAI2025,ghosh2025mammofmbreastspecificfoundationalmodel}, a clinically vital modality where breast cancer accounts for approximately 42,000 deaths per year in the United States alone~\cite{siegel2018cancer} and regular screening has been shown to reduce mortality by 38--48\%~\cite{gotzsche2013screening,broeders2012impact}.
A defining characteristic of mammography is its multi-view nature: each exam typically includes four images—two views (craniocaudal (CC), and mediolateral oblique, (MLO)) per breast—interpreted jointly to produce a single report. To date, most mammography VLMs have focused almost exclusively on this multi-view characteristic, using view-specific encoders~\cite{chen2024mammoclipleveragingcontrastivelanguageimage,du2025multiviewmultiscalealignmentcontrastive,ghosh2025mammofmbreastspecificfoundationalmodel} and cross-view correspondence learning~\cite{DuYue_GeometryGuided_MICCAI2025}. However, this multi-view challenge is only part of the problem. We identify two more fundamental limitations overlooked by prior work:

\begin{figure}[t]
     \centering
         \includegraphics[width=1.0\textwidth]{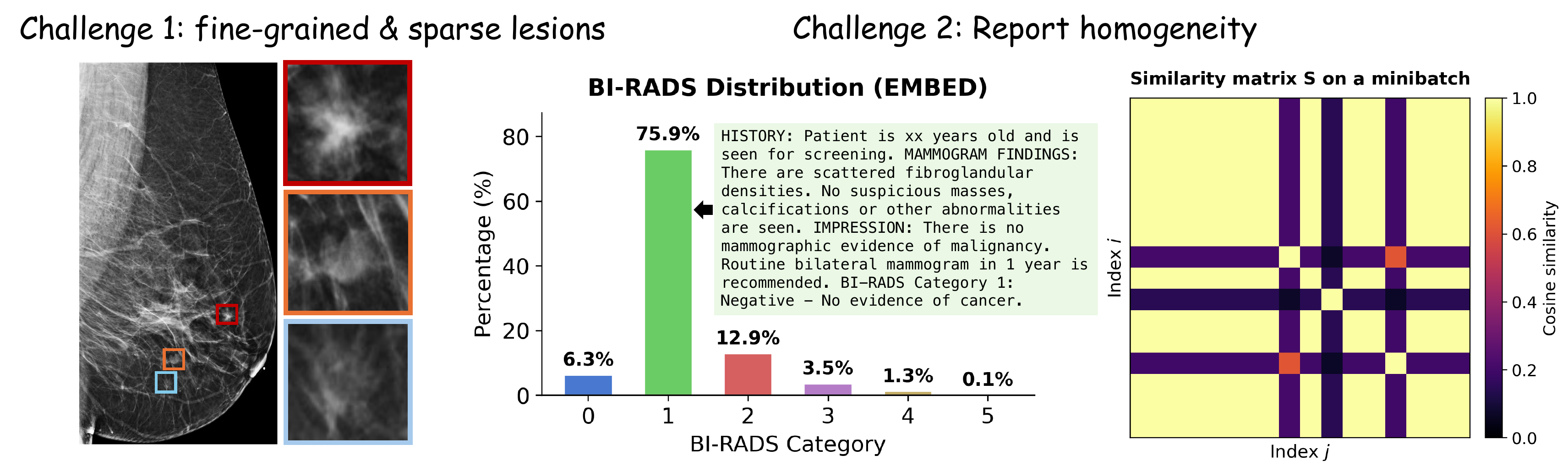}
{\caption{Existing mammography VLMs overlook two key properties of mammography data. \textbf{(1) High resolution with sparse findings:} mammograms contain small, sparsely located lesions that are easily destroyed by downsampling, yet retaining full resolution forces a prohibitive batch size tradeoff. \textbf{(2) Report homogeneity:} the screening-dominant nature of mammography produces highly similar reports, causing the contrastive loss to falsely repel semantically similar image--report pairs.}}
\label{fig:motivation}
\end{figure}

\paragraph{Challenge 1: High resolution and sparse findings.}
% % small 
Mammograms are acquired at much higher resolutions than the natural images used in standard CLIP training ($224^2$–$512^2$ pixels), often reaching over 12 megapixels. This resolution is clinically important (Fig.~\ref{fig:motivation}(a)), as many lesions in breasts are often small (<5\% of the image) or obscured by tissue overlaps, especially in dense breasts~\cite{boyd2007mammographic}. Downsampling risks erasing these fine-grained features, yet directly applying the standard CLIP framework at full resolution is not viable: the resulting reduction in batch size increases training time and can degrade performance, since contrastive loss benefits from large-batch training~\cite{radford2021learningtransferablevisualmodels,chen2020simpleframeworkcontrastivelearning,he2020momentumcontrastunsupervisedvisual}.
% % sparse
Furthermore, the lesions are not only small but also sparse, with a typical exam containing only a handful when present. The mammography report, which we aim to align with the mammogram images, primarily describes these findings. This differs from the image--caption pairs used in standard CLIP pretraining, where captions often describe clearly visible objects in images. If  lesion locations are known in advance, the majority part of the image can be safely discarded, retaining only regions described in the report. A potential approach is a two-step pipeline: (1) a Region Proposal Network~\cite{ren2016fasterrcnnrealtimeobject} to localize findings, followed by (2) CLIP pretraining on the selected patches. However, this is impractical for most mammography datasets~\cite{jeong2023emory,nguyen2023vindrmammolargescalebenchmarkdataset,rsna2022breastcancer}, where spatial annotations are largely unavailable or insufficient to train a reliable proposal network.

\paragraph{Challenge 2: Report homogeneity.}
In the US, women are recommended to undergo annual screening mammography beginning at age 40~\cite{uspstf2024breast}, with approximately 90\% of exams yielding normal or benign findings. For example, in the EMBED dataset~\cite{jeong2023emory}, 75\% of exams are BI-RADS 1 (Negative), and these reports differ only in density and demographics (Fig.~\ref{fig:motivation}(b)). In contrast, standard CLIP training data~\cite{schuhmann2022laion5bopenlargescaledataset,Thomee_2016} exhibits high caption diversity, and the contrastive loss is specifically designed to exploit this: the softmax normalization, paired with an identity ground-truth matrix, assumes each caption is semantically distinct from all others in the batch. This assumption is violated in mammography, where reports are largely homogeneous. As Fig.~\ref{fig:motivation}(b) illustrates, the semantic similarity of reports in a minibatch is pervasive (see Sec.~\ref{sec:soft-sigmoid} for the derivation of similarity matrix), causing the contrastive loss to falsely repel semantically similar image–report pairs. While prior work has addressed related issues via supervised~\cite{khosla2020supervised,yang2022unified,wang2022medclip} and hard-negative mining approaches~\cite{robinson2021contrastivelearninghardnegative,kalantidis2020hard}, they are not tailored to this extreme low-diversity setting, where the fixes are insufficient.

\paragraph{Contributions.}

% method contribution
We propose \textbf{TopKSigLIP}, a VLM for mammography that addresses two underexplored challenges.
To handle high resolution and sparse findings, we introduce \textbf{TopK-Patch}, a differentiable top-$k$ patch selection module that predicts sampling scores from low-res inputs and selects $K$ high-res patches likely to contain lesions. This preserves fine-grained detail while discarding irrelevant regions, and the learned scores/locations provide built-in localization capability, replacing post-hoc methods such as Grad-CAM~\cite{selvaraju2016grad}.
To address report homogeneity, we propose a \textbf{Sup-sigmoid} loss, which extends the sigmoid loss from SigLIP~\cite{zhai2023sigmoid} with soft labels derived from structured clinical data (e.g., MagView~\cite{magview2026}). While the sigmoid loss was originally introduced to reduce the computational burden of the contrastive loss, we repurpose it to mitigate a training bias: softmax-based supervised contrastive objectives~\cite{wang2022medclip,yang2022unified} force each pair to compete against all others in the batch, even when most are semantically similar, as is the case in mammography. The sigmoid formulation instead treats each pair independently, alleviating this unnecessary competition.

% result contribution
Trained on a subset of the EMBED dataset~\cite{jeong2023emory}, TopKSigLIP achieves state-of-the-art performance on internal and external benchmarks (VinDr~\cite{nguyen2023vindrmammolargescalebenchmarkdataset}, RSNA~\cite{rsna2022breastcancer}) across BI-RADS classification, density estimation, lesion subtyping, and cancer prediction, in both zero-shot and linear probing settings. TopK-Patch enables memory-efficient training with large batch sizes on a single GPU even at full resolution, while its built-in localization outperforms Grad-CAM for lesion detection.

\section{Related Works}
\textbf{Mammography VLMs.} Two independently proposed Mammo-CLIP models~\cite{ghosh2024mammoclipvisionlanguagefoundation,chen2024mammoclipleveragingcontrastivelanguageimage} mark the first VLM efforts for mammography. Chen et al.~\cite{chen2024mammoclipleveragingcontrastivelanguageimage} aggregate all four views into a single exam-level embedding, whereas Ghosh et al.~\cite{ghosh2024mammoclipvisionlanguagefoundation} operate on a single view and introduce Mammo-FActOR for post-hoc lesion localization. MaMA~\cite{du2025multiviewmultiscalealignmentcontrastive} leverages two views (CC and MLO) with a local contrastive loss for fine-grained features, and GLAM~\cite{DuYue_GeometryGuided_MICCAI2025} incorporates cross-view attention. More recently, Ghosh et al.~\cite{ghosh2024mammoclipvisionlanguagefoundation} extend this line to report generation~\cite{ghosh2025mammofmbreastspecificfoundationalmodel}. However, these works primarily address multi-view learning and do not tackle the high-res or report homogeneity challenges we target.
\textbf{Differentiable top-$k$ learning.}
Top-$k$ sampling, the task of stochastically selecting a subset of size $K$ from a collection based on learned importance scores, can be made differentiable through several approaches. One uses entropy-regularized optimal transport~\cite{cuturi2013sinkhorn}, treating the collection as source and the top-$k$ indicator as target~\cite{xie2020differentiable}. Another relies on perturbed optimization~\cite{berthet2020learning,cordonnier2021differentiable}, adding noise to a linear program and approximating gradients via Monte Carlo. We adopt a Gumbel-softmax approach~\cite{kool2019stochastic,xie2019reparameterizable,jeon2025no} built on weighted reservoir sampling~\cite{efraimidis2006weighted}.
\textbf{CLIP alternatives.}
While CLIP~\cite{radford2021learningtransferablevisualmodels} is a powerful VLM pretraining paradigm, it suffers from penalizing semantically similar pairs as negatives and large memory footprints. To mitigate this, MedCLIP~\cite{wang2022medclip} and UniCL~\cite{yang2022unified}, inspired by supervised contrastive learning~\cite{khosla2020supervised}, augment the identity ground-truth similarity matrix with soft/hard targets derived from raw text or tabular data. In contrast, training efficiency from large memory footprint has received less attention: prior work either freezes the image encoder~\cite{zhai2022lit} or replaces the softmax in the contrastive loss with a sigmoid~\cite{zhai2023sigmoid,tschannen2025siglip}, but none address the memory bottleneck of high-res inputs in a fully end-to-end trainable setting.

\begin{figure}[t]
     \centering
         \includegraphics[width=0.99\textwidth]{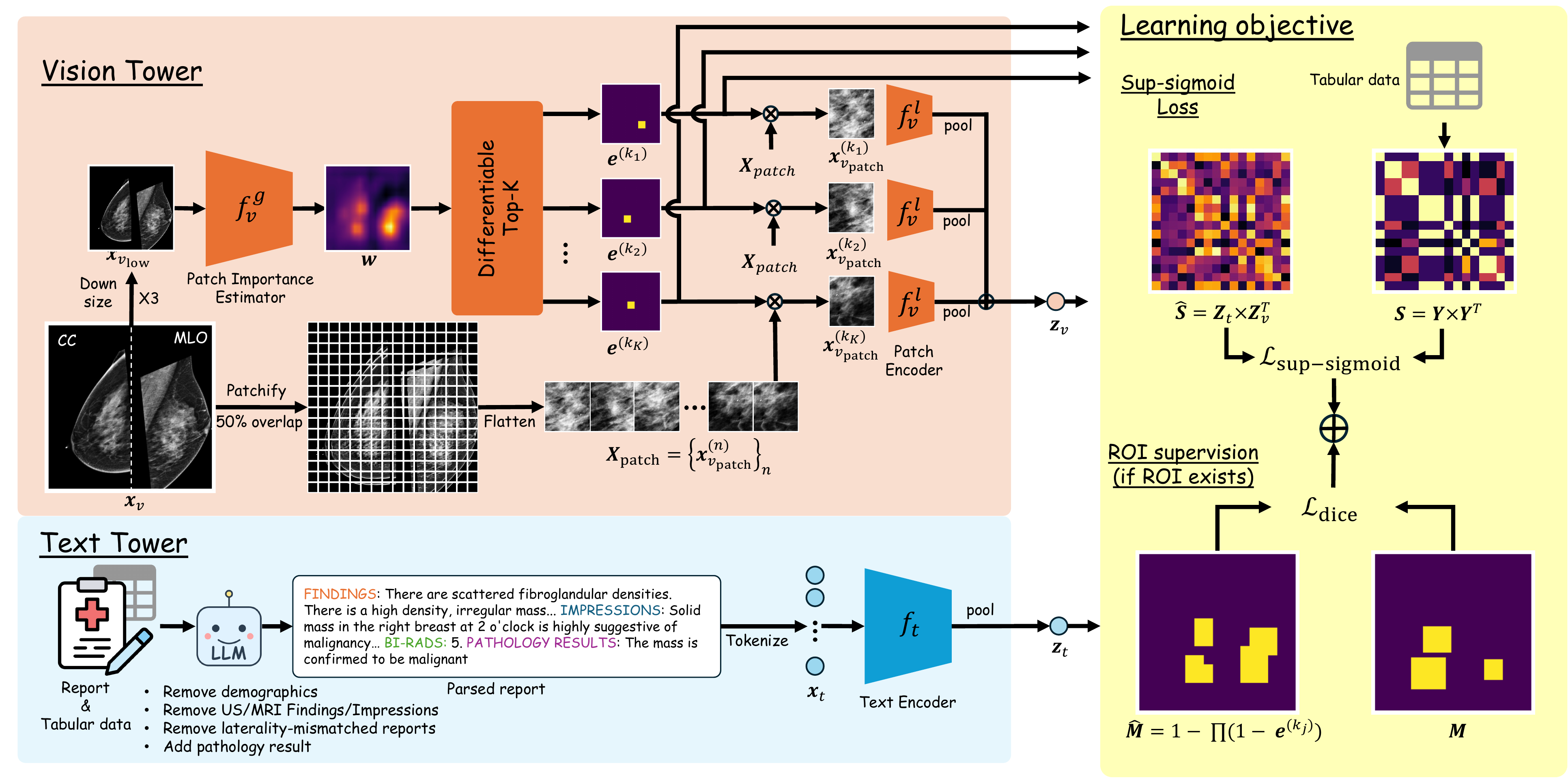}
{\caption{\textbf{Overall training pipeline of TopKSigLIP.} The model comprises a Vision Tower, which selects $K$ high-res patches that contain suspicious lesions via a TopK-Patch module based on a low-res image, and a Text Tower, which encodes LLM-preprocessed radiology reports into text embeddings. The image and text embeddings are jointly optimized using three objectives: a Sup-sigmoid loss supervised by tabular-derived semantic similarity and ROI supervision via Soft Dice when bounding-box annotations are available.}
}
\label{fig:main}
\end{figure}

\section{Method}

\subsection{Overview}
TopKSigLIP (Fig.~\ref{fig:main}) builds on the standard CLIP with two key modifications: the image tower and the learning objective. Given a minibatch of image--report pairs $\{(\mathbf{x}_v^{(i)} \in \mathbb{R}^{H \times W},\ \mathbf{x}_t^{(i)} \in \mathbb{Z}^{L})\}_{i=1}^{B}$, a vision tower and a text tower produce $\ell_2$-normalized embeddings $\{(\mathbf{z}_v^{(i)},\ \mathbf{z}_t^{(i)}) \in \mathbb{R}^{C} \times \mathbb{R}^{C}\}$. The all-pair similarity matrix $\hat{\mathbf{S}} \in [-1,1]^{B\times B}$, with $\hat{S}_{ij} = {\mathbf{z}_v^{(i)}}^\top \mathbf{z}_t^{(j)}$, is optimized to predict a label matrix $\mathbf{S} \in [0,1]^{B \times B}$. We detail the image tower and learning objectives below. The text tower follows a standard transformer architecture, with modifications limited to report preprocessing (Sec.~\ref{sec:implementation_details}).

\subsection{Vision tower}
The vision tower takes as input a high-res image $\mathbf{x}_v$, formed by vertically concatenating the CC and MLO views of a single breast. For brevity, we omit the batch index $(i)$ unless otherwise stated. The vision tower produces an image embedding $\mathbf{z}_v$ in two steps: (1) the TopK-Patch module selects $K$ high-res patches $\{\mathbf{x}_{\mathrm{patch}}^{(k_j)} \in \mathbb{R}^{H_p \times W_p}\}_{j=1}^K$ from $N$ candidates based on importance scores $\mathbf{w} \in \mathbb{R}_{\geq 0}^{N}$ estimated from a low-res image $\mathbf{x}_{v_{\mathrm{low}}}$, obtained by downsizing $\mathbf{x}_v$ by a predefined factor; (2) a patch encoder $f_v^l$ maps the $K$ patches to embeddings $\{\mathbf{z}_{\mathrm{patch}}^{(k_j)}\}_{j=1}^K$, which are aggregated to yield $\mathbf{z}_v$.

\subsubsection{TopK-Patch: Differentiable Top-$k$ Patch Selection}

\paragraph{Differentiable top-$k$.}
An encoder $f_v^g$ maps the low-res image $\mathbf{x}_{v_{\mathrm{low}}}$ to patch importance scores $\mathbf{w} \in \mathbb{R}_{\geq 0}^{N}$ over all $N = N_h \times N_w$ patch locations, where $N_h$ and $N_w$ are the number of patch candidates on each spatial axis. Based on $\mathbf{w}$, we wish to sample $K$ locations represented by a one-hot vector sequence $E = (\mathbf{e}^{(k_1)}, \ldots, \mathbf{e}^{(k_K)})$, where $\mathbf{e}^{(k_j)} \in \{0,1\}^N$ with 1 at the $k_j$-th index. Sampling without replacement with probabilities proportional to $\mathbf{w}$, $E$ follows a Plackett--Luce distribution~\cite{plackett1975analysis}:
\begin{equation}
    p(E \mid \mathbf{w}) = 
    \frac{w_{k_1}}{Z} \cdot \frac{w_{k_2}}{Z - w_{k_1}} \cdots 
    \frac{w_{k_K}}{Z - \sum_{j=1}^{K-1} w_{k_j}},
    \label{eq:wrs}
\end{equation}
where $Z = \sum_{i=1}^{N} w_i$. However, establishing gradient flow through $\mathbf{w}$ is non-trivial due to both the stochastic nature of sampling~\cite{kingma2013auto,williams1992simple} and discreteness of the events~\cite{jang2017categoricalreparameterizationgumbelsoftmax,maddison2017concretedistributioncontinuousrelaxation}. We resolve this by leveraging the Gumbel-top-$K$ reparameterization~\cite{vieira2014gumbel,xie2019reparameterizable}: exact samples from $p(E \mid \mathbf{w})$ are obtained by perturbing the log-weights with Gumbel noise:
\begin{equation}
    \hat{r}_i = \log(w_i) + g_i, \quad g_i \sim \mathrm{Gumbel}(0,1),
\end{equation}
and selecting the top-$K$ elements of $\hat{\mathbf{r}}$. Since the stochastic noise $g_i$ is independent of $\mathbf{w}$, it is treated as a constant during backpropagation, restoring gradient flow through $\mathbf{w}$. The top-$k$ operation is then relaxed via an iterative temperature annealed softmax~\cite{plotz2018neural}. Starting from $\alpha_i^{(k_1)} := \hat{r}_i$, at each step $j = 1, \ldots, K$:
\begin{equation}
    e_i^{(k_j)} = \frac{\exp(\alpha_i^{(k_j)} / \tau)}{\sum_{m=1}^{N} \exp(\alpha_m^{(k_j)} / \tau)}, 
    \qquad
    \alpha_i^{(k_{j+1})} := \alpha_i^{(k_j)} + \log\left(1 - e_i^{(k_j)}\right),
    \label{eq:softopk}
\end{equation}
where $\tau > 0$ is a temperature parameter. A proof that the Gumbel-top-$K$ reparameterization yields exact samples from the Plackett--Luce distribution is provided in Appendix~\ref{sec:proof_top_k}.

\paragraph{Image embedding extraction.}
Given the sampled patch location sequence $E$, we extract the corresponding patches from patch candidates $\mathbf{X}_{\mathrm{patch}} = \{\mathbf{x}_{v_\mathrm{patch}}^{(n)} \in \mathbb{R}^{P_h \times P_w}\}_{n=1}^{N}$, obtained by patchifying $\mathbf{x}_{v}$ with a predefined patch size $P_h \times P_w$ and overlap ratio, via the inner product:
\begin{equation}
    \mathbf{x}_{v_\mathrm{patch}}^{(k_j)} = \langle \mathbf{e}^{(k_j)}, \mathbf{X}_{\mathrm{patch}} \rangle, \; \forall j \in \{1 ... K\}.
\end{equation}
The extracted patches are mapped to embeddings via $f_v^l$, i.e., $\mathbf{z}_{\mathrm{patch}}^{(k_j)} = f_v^l(\mathbf{x}_{\mathrm{patch}}^{(k_j)})$, and aggregated into a single image embedding with simple averaging $\mathbf{z}_v = \frac{1}{K}\sum_{j=1}^{K} \mathbf{z}_{\mathrm{patch}}^{(k_j)}$. Other pooling methods such as self-attention with [cls] token can be considered for potentially better performance.

\subsection{Learning objectives}

\subsubsection{Sup-sigmoid loss}
\label{sec:soft-sigmoid}

Given the homogeneity of mammography reports, the image-report similarity matrix $\hat{\mathbf{S}}$ in a mini-batch should reflect high semantic similarity broadly across pairs, not only on the diagonal. While the sigmoid loss~\cite{zhai2023sigmoid} affords the flexibility to learn such a non-diagonal similarity matrix by treating each pair independently rather than forcing saturation through partition function, when paired with the standard identity target matrix $\mathbf{S} = \mathbf{I}$, this flexibility is wasted.

We replace $\mathbf{I}$ with a similarity matrix derived from tabular data. Specifically, for the $i$-th sample in the minibatch, a ground-truth semantic vector $\mathbf{y}^{(i)}$ is constructed by encoding key tabular attributes --- BI-RADS score, breast density, cancer outcome, and lesion subtypes --- as individual one-hot or multi-hot vectors, which are then concatenated to form a single vector. The similarity between the $i$-th image and $j$-th report is then computed as $S_{i,j} = \langle \mathbf{y}^{(i)}, \mathbf{y}^{(j)} \rangle / \left( \| \mathbf{y}^{(i)} \| \cdot \| \mathbf{y}^{(j)} \| \right)$.

Taking the predicted similarity $\hat{\mathbf{S}}$ and the ground-truth similarity $\mathbf{S}$ as inputs, the sup-sigmoid loss is defined as:
\begin{equation}
\mathcal{L}_{\mathrm{sup\text{-}sigmoid}} = -\frac{1}{B} \sum_{i=1}^{B} \sum_{j=1}^{B} 
\left[ S_{i,j} \log \sigma_\gamma(\hat{S}_{i,j}) + 
(1 - S_{i,j}) \log\left(1 - \sigma_{\gamma}(\hat{S}_{i,j})\right) \right],
\label{eq:sup-sigmoid}
\end{equation}
where $\sigma_{\gamma}(\cdot)$ denotes the sigmoid function with temperature parameter $\gamma$.

\subsubsection{ROI supervision}
\label{sec:roi-supervision}
Providing region-of-interest (ROI) supervision can accelerate the learning of TopK-Patch and improve accuracy. In the EMBED dataset, approximately 4\% of exams include at least one ROI annotation, which we incorporate into training whenever available.

The predicted ROI mask $\hat{\mathbf{M}} \in [0,1]^{N_h \times N_w}$ is computed by aggregating the sampled patch locations $E = \{\mathbf{e}^{(k_j)}\}$. The ground-truth ROI mask $\mathbf{M} \in \{0,1\}^{N_h \times N_w}$ is generated by converting bounding-box coordinates into a binary mask. A Soft Dice is minimized between $\mathbf{M}$ and $\hat{\mathbf{M}}$:
\begin{equation}
    \mathcal{L}_{\text{dice}} = 1 - \frac{2\sum_{h,w} M_{h,w} \cdot \hat{M}_{h,w}}{\sum_{h,w} M_{h,w} + \sum_{h,w} \hat{M}_{h,w}}, \quad
    \hat{\mathbf{M}} = \mathrm{Reshape}_{N_h \times N_w} \Bigg(1 - \prod_{j=1}^K \big( 1 - \mathbf{e}^{(k_j)} \big) \Bigg),
    \label{eq:dice_loss}
\end{equation}
where $\mathrm{Reshape}_{N_h \times N_w}(\cdot)$ reshapes the resulting vector into an $N_h \times N_w$ mask.

\subsubsection{Final objective}
The final training objective is a weighted sum of Sup-sigmoid and Soft Dice losses:
\begin{equation}
    \mathcal{L} = \mathcal{L}_{\mathrm{sup\text{-}sigmoid}} + \lambda_{\mathrm{dice}}\, \mathcal{L}_{\mathrm{dice}},
\end{equation}
where $\lambda_{\mathrm{dice}}$ is a scalar weight. $\mathcal{L}_{\mathrm{dice}}$ is applied only to the ${\sim}4\%$ of samples with ROI labels. The batch average for $\mathcal{L}_{\mathrm{dice}}$ is omitted for the sake of brevity.

\section{Experiments}

\subsection{Datasets}
TopKSigLIP is trained on EMBED~\cite{jeong2023emory} and evaluate on EMBED (internal), VinDr~\cite{nguyen2023vindrmammolargescalebenchmarkdataset}, and RSNA~\cite{rsna2022breastcancer} (external).
EMBED~\cite{jeong2023emory} has 364,000 exams from the US. We use cohorts 1--2 (openly available) for training, 8 for validation, and 9--10 for testing. Each exam has multi-view images (primarily CC and MLO, with occasional special views for diagnostics), tabular results, radiology report, and ROI annotations for 4\% of exams.
VinDr~\cite{nguyen2023vindrmammolargescalebenchmarkdataset} has 5,000 exams from Vietnam with BI-RADS and density labels, as well as ROIs for non-benign findings. We use the predefined 4,000/1,000 train/test split.
RSNA~\cite{rsna2022breastcancer} has 11,913 screening exams from the US and Australia with cancer outcome labels and BI-RADS scores. No ROI or finding labels are provided. We use an 80/20 train/test split.
For both external datasets, the training sets are used only for linear probe training. 20\% of the training set is held out for validation.
We provide the label statistics of each dataset in Appendix~\ref{sec:a_label_statistics}. 

\subsection{Baselines}
We compare with four mammography VLMs --- Mammo-CLIP-B2/B5~\cite{ghosh2024mammoclipvisionlanguagefoundation}, MaMA~\cite{du2025multiviewmultiscalealignmentcontrastive}, and GLAM~\cite{DuYue_GeometryGuided_MICCAI2025} --- and one general-purpose medical VLM: MII~\cite{codella2024medimageinsightopensourceembeddingmodel}. 
We follow the model's proposed image, text pre-processing as well as their  zero-shot prompts where available (see Appendix~\ref{sec:a_evaluation_protocol_zero_shot}). For single-view models (Mammo-CLIP-B2/B5, MII), multi-view embeddings are obtained by averaging per-view output. Mammo-CLIP are trained on UPMC (their internal dataset) and VinDR. Identical to TopKSigLIP, MaMA and GLAM are trained on EMBED's cohorts 1\&2. MII is trained on RSNA and other non-breast radiology images. Thus we exclude MII in RSNA evaluation.

\subsection{Implementation details}
\label{sec:implementation_details}

\paragraph{preprocessing.}
CC and MLO views per breast are resized to $1536 \times 768$ after background removal and vertically concatenated to have $1536 \times 1536$. When standard views are unavailable (mostly diagnostic exams), other available views are substituted.
Reports are processed with LlaMA-3\cite{grattafiori2024llama3herdmodels} to extract mammography-specific findings, excluding patient-specific metadata such as age, race, and medical history. We further exclude findings from non-mammography modalities such as US and MRI, as they can not be inferred from mammogram images. Since the vision tower processes paired CC and MLO views from a single breast, only findings matching the corresponding laterality are retained. See Appendix~\ref{sec:a_report_parsing} for the LLM prompt and output examples. Cancer outcome from pathology is appended to support cancer risk prediction.

\paragraph{Architecture.}
The two vision encoders $f_v^g$ and $f_v^l$ are weight-shared, ImageNet-pretrained ConvNeXt-Tiny~\cite{liu2022convnet}. $f_v^g$ is truncated by dropping the last convolutional block. The text encoder $f_t$ is the pretrained CLIP text encoder with context length extended to 170. $\mathbf{x}_{v_{\mathrm{low}}}$ is obtained by downsizing $\mathbf{x}_v$ by $\times3$, resulting in image size $512 \times 512$. The patch grid is set to $N = N_h \times N_w = 16 \times 16 = 256$, with a patch size of $181 \times 181$ and 50\% overlap. The output of $f_v^g$ is resized and flattened so that $\mathbf{w}$ has 256 entries. The number of patches sampled is set to $K = 15$ during training.

\paragraph{Training.}
The temperature $\tau$ in the TopK-Patch starts at 1 and is gradually annealed to 2/3. The temperature $\gamma$ in Sup-sigmoid loss is learnable with an initial value of 10. $\lambda_{\mathrm{dice}}$ is set to 5. TopKSigLIP is trained with AdamW~\cite{loshchilov2017decoupled} for 100,000 iterations (learning rate $1\times10^{-5}$, weight decay 0.05, batch size 32) on a single RTX PRO 6000. Image augmentations include random affine transforms, intensity scaling, and ROI size and location perturbation.

\subsection{Evaluation protocol}

\paragraph{Accuracy.}
Model accuracy is accessed with AUROC metric across four tasks: BI-RADS classification, density classification, finding type identification --- mass (MASS), calcification (CALC), architectural distortion (AD), and asymmetry (AS) --- and cancer outcome prediction, in both zero-shot and linear probing settings. BI-RADS categories 0 and 6 are excluded from the BI-RADS classification task. 
For zero-shot evaluation, we adopt each model's published prompt template where available.
Linear probe evaluation attaches a single fully-connected layer on a frozen image encoder and trained per task per dataset minimizing class-weight adjust cross entropy loss with AdamW (20,000 iterations, learning rate $4\times10^{-3}$, weight decay 0.01, batch size 32). 
Further details on the evaluation protocols, including the rationale for excluding BI-RADS 0 and 6 and the zero-shot prompts used, are provided in Appendix~\ref{sec:a_evaluation_protocol}.

\paragraph{Localization.}
We evaluate localization using pixel-level average precision (AP) and Pointing Game (PG)~\cite{zhang2016topdownneuralattentionexcitation} between predicted heatmaps and ground-truth ROIs. TopKSigLIP uses its built-in patch importance scores $\mathbf{w}$ directly as the heatmap. All baseline models lack a built-in localization mechanism; for these, we apply \textit{Oracle Grad-CAM}, which uses the ground-truth label to guide heatmap generation in both zero-shot and linear probe settings. In the zero-shot setting, the ground-truth guided logit is computed as the inner product between the image embedding and the text embedding of the ground-truth prompt. In the linear probe setting, the logit outputs from the linear layer are masked by the ground-truth label and summed to form a single scalar logit. In both cases, gradients of this logit with respect to the image encoder's feature map prior to spatial pooling are used to produce the Grad-CAM heatmap. Appendix.~\ref{sec:a_localization_evaluation} provides further detail on the \textit{Oracle Grad-CAM} implementation in both settings.

\begin{table*}[t]
\centering
\caption{Zero-shot performance (AUC). \textbf{Bold}: best, \underline{underline}: second best. $k$ denotes the number of patches sampled at inference in TopKSigLIP. ``--'' indicates missing ground truth labels.}
\label{tab:zeroshot}
\definecolor{datasetgray}{gray}{0.90}
\resizebox{\textwidth}{!}{%
\begin{tabular}{l|ccccc|c|cccccc|ccccc}
\toprule
& \multicolumn{5}{c|}{Findings}
& Cancer
& \multicolumn{6}{c|}{BI-RADS}
& \multicolumn{5}{c}{Density} \\
\cmidrule(lr){2-6} \cmidrule(lr){7-7} \cmidrule(lr){8-13} \cmidrule(lr){14-18}
& Mass
& Calc
& AD
& AS
& Avg
& {}
& 1
& 2
& 3
& 4
& 5
& Avg
& A
& B
& C
& D
& Avg \\
\midrule
%
% ==================== EMBED ====================
\rowcolor{datasetgray}
\multicolumn{18}{c}{\textbf{EMBED}} \\
\hline
\noalign{\vspace{4pt}}
MedImageInsight   & 62.1 & 54.0 & 39.9 & 48.7 & 51.2 & 67.6 & 46.7 & 56.5 & 61.8 & 56.2 & 64.7 & 57.2 & 76.6 & 41.8 & 65.5 & 47.2 & 57.7 \\
MammoCLIP-B2      & 41.1 & 47.8 & 58.4 & 44.4 & 47.9 & 50.7 & 59.0 & 59.8 & 55.7 & 47.2 & 44.2 & 53.2 & 94.3 & 80.9 & 81.8 & 93.0 & 87.5 \\
MammoCLIP-B5      & 43.5 & 33.5 & 53.5 & 56.6 & 46.8 & 46.4 & 64.1 & 60.0 & 57.7 & 64.0 & 58.4 & 60.8 & 94.8 & 84.4 & 77.2 & 91.8 & 87.0 \\
MaMA              & 69.7 & 59.7 & 53.0 & 58.0 & 60.1 & 49.4 & 67.6 & 62.1 & 69.1 & 37.0 & 80.5 & 63.3 & 95.8 & 85.7 & 91.4 & 96.0 & 92.2 \\
GLAM              & 54.2 & 50.2 & 49.7 & 53.3 & 51.9 & 47.6 & 58.5 & 62.1 & 59.0 & 60.4 & 35.4 & 55.1 & 87.1 & 74.7 & 76.9 & 82.0 & 80.2 \\
\hline
\noalign{\vspace{4pt}}
TopKSigLIP (k=15) & 81.9             & 79.1             & \textbf{64.0}    & \textbf{67.2}    & 73.1             & 85.0             & 80.8             & 77.6             & 78.6             & 86.9             & 77.0             & 80.2             & \underline{96.3} & \textbf{91.1}    & \textbf{93.7}    & \textbf{97.1}    & \textbf{94.6}    \\
TopKSigLIP (k=30) & \textbf{82.4}    & \underline{80.4} & \underline{63.9} & \underline{67.1} & \textbf{73.4}    & \textbf{85.4}    & \underline{82.3} & \underline{78.8} & \underline{79.7} & \underline{88.4} & \underline{81.2} & \underline{82.1} & 96.3             & \underline{91.1} & \underline{93.6} & \underline{97.1} & \underline{94.5} \\
TopKSigLIP (k=45) & \underline{82.1} & \textbf{80.6}    & 63.2             & 66.6             & \underline{73.1} & \underline{85.2} & \textbf{82.7}    & \textbf{79.2}    & \textbf{79.8}    & \textbf{88.7}    & \textbf{82.6}    & \textbf{82.6}    & \textbf{96.3}    & 91.1             & 93.6             & 97.1             & 94.5             \\
\midrule
%
% ==================== VinDr ====================
\rowcolor{datasetgray}
\multicolumn{18}{c}{\textbf{VinDr}} \\
\hline
\noalign{\vspace{4pt}}
MedImageInsight   & 58.8             & 59.7             & 39.7             & 63.1             & 55.3             & --                & 40.2             & 47.4             & 49.1             & 44.2             & 79.9             & 52.1             & 83.0          & 28.6             & 54.8             & 45.7             & 53.0             \\
MammoCLIP-B2      & 58.6             & 66.0             & 54.0             & 58.3             & 59.2             & --                & 65.0             & \textbf{68.8}    & 49.2             & 46.5             & 45.6             & 55.0             & 85.2          & 94.3             & 71.0             & 77.1             & 81.9             \\
MammoCLIP-B5      & 47.0             & 75.6             & 61.6             & 43.9             & 57.0             & --                & 71.8             & 59.8             & 50.4             & 60.5             & 88.3             & 66.1             & \textbf{98.4} & \underline{94.9} & 65.5             & 82.3             & 85.3             \\
MaMA              & 64.8             & 64.2             & 42.4             & 30.8             & 50.6             & --                & 65.2             & 51.5             & 55.6             & 61.8             & \underline{90.7} & 65.0             & 94.1          & \textbf{96.3}    & \textbf{85.1}    & 84.9             & \underline{90.1} \\
GLAM              & 54.9             & 45.4             & 58.2             & 64.5             & 55.8             & --                & 49.2             & 50.4             & 54.0             & 50.4             & 31.8             & 47.2             & 95.6          & 80.7             & 72.4             & 52.5             & 75.3             \\
\hline
\noalign{\vspace{4pt}}
TopKSigLIP (k=15) & 81.2             & 97.9             & \textbf{69.1}    & \underline{71.0} & \textbf{79.8}    & --                & 72.6             & 63.9             & \underline{69.7} & \textbf{77.9}    & 84.5             & 73.7             & 97.7             & 92.3 & \underline{81.1} & \underline{89.8} & \textbf{90.2}    \\
TopKSigLIP (k=30) & \textbf{81.5}    & \underline{98.3} & \underline{66.9} & \textbf{71.2}    & \underline{79.5} & --                & \textbf{73.2}    & \underline{64.5} & \textbf{70.1}    & 77.3             & \textbf{90.8}    & \textbf{75.2}    & \underline{97.7} & 92.1 & 80.8             & 89.8             & 90.1             \\
TopKSigLIP (k=45) & \underline{81.4} & \textbf{98.6}    & 64.2             & 66.5             & 77.7             & --                & \underline{73.1} & 64.2             & 68.7             & \underline{77.9} & 90.0             & \underline{74.8} & 97.7             & 92.1 & 80.7             & \textbf{89.8}    & 90.0             \\
\midrule
%
% ==================== RSNA ====================
\rowcolor{datasetgray}
\multicolumn{18}{c}{\textbf{RSNA}} \\
\hline
\noalign{\vspace{4pt}}
MammoCLIP-B2      & -- & -- & -- & -- & -- & 62.2          & 67.1             & 63.7             & -- & -- & -- & 65.4             & 94.9             & 82.8             & 79.2             & 94.5             & 87.8 \\
MammoCLIP-B5      & -- & -- & -- & -- & -- & 59.9          & 68.8             & 63.5             & -- & -- & -- & 66.1             & 95.4             & 85.7             & 74.2             & 92.6             & 87.0 \\
MaMA              & -- & -- & -- & -- & -- & 64.1          & 75.8             & 61.2             & -- & -- & -- & 68.5             & 95.6             & 85.6             & 89.7             & 95.9             & 91.7 \\
GLAM              & -- & -- & -- & -- & -- & 55.7          & 52.8             & 65.6             & -- & -- & -- & 59.2             & 88.2             & 76.2             & 76.8             & 84.4             & 81.4 \\
\hline
\noalign{\vspace{4pt}}
TopKSigLIP (k=15) & -- & -- & -- & -- & -- & \textbf{78.7}    & 77.1             & 77.7             & -- & -- & -- & 77.4             & \underline{96.4} & \textbf{90.6}    & \textbf{92.1}    & \textbf{97.1}    & \textbf{94.1}    \\
TopKSigLIP (k=30) & -- & -- & -- & -- & -- & \underline{78.2} & \underline{78.7} & \underline{78.9} & -- & -- & -- & \underline{78.8} & 96.4             & \underline{90.6} & \underline{92.1} & 97.0             & \underline{94.0} \\
TopKSigLIP (k=45) & -- & -- & -- & -- & -- & 77.6             & \textbf{79.5}    & \textbf{79.4}    & -- & -- & -- & \textbf{79.5}    & \textbf{96.4}    & 90.6             & 92.0             & \underline{97.1} & 94.0             \\
\bottomrule
\end{tabular}%
}
\end{table*}
\begin{table*}[t]
\centering
\caption{linear probe classification performance (AUC). \textbf{Bold}: best, \underline{underline}: second best. ``--'' indicates missing ground truth labels.}
\label{tab:linearprobe}
\definecolor{datasetgray}{gray}{0.90}
\resizebox{\textwidth}{!}{%
\begin{tabular}{l|cccc|cccc|cccc}
\toprule
& \multicolumn{4}{c|}{\textbf{EMBED}}
& \multicolumn{4}{c|}{\textbf{VinDr}}
& \multicolumn{4}{c}{\textbf{RSNA}} \\
\cmidrule(lr){2-5} \cmidrule(lr){6-9} \cmidrule(lr){10-13}
& Findings & Cancer & BI-RADS & Density
& Findings & Cancer & BI-RADS & Density
& Findings & Cancer & BI-RADS & Density \\
\midrule
MedImageInsight   & 63.0             & 85.0             & \underline{86.5} & 94.5             & 58.7             & --  & \textbf{83.7}    & \textbf{94.1}    & --  & --               & --               & --               \\
MammoCLIP-B2      & 66.1             & 83.6             & 85.2             & 94.5             & 62.0             & --  & 77.9             & 93.2             & --  & 75.8             & 76.7             & 93.4             \\
MammoCLIP-B5      & 61.5             & 85.0             & 85.6             & \underline{94.6} & 61.5             & --  & 79.2             & 93.2             & --  & 76.9    & \underline{79.7} & 93.3             \\
MaMA              & 73.6             & \underline{86.0} & \textbf{87.3}    & \textbf{94.9}    & 72.2             & --  & 77.0             & \underline{94.0} & --  & 75.6             & \textbf{83.6}    & \textbf{94.4}    \\
GLAM              & 68.2             & 77.7             & 82.1             & \underline{94.6} & 69.2             & --  & 70.4             & 93.1             & --  & 63.2             & 71.7             & 93.8             \\
\hline
\noalign{\vspace{4pt}}
TopKSigLIP (k=15) & 73.6             & 85.5             & 85.1             & \underline{94.6} & 80.1             & --  & 81.8             & \underline{94.0} & --  & \textbf{78.3}    & 76.9    & \underline{94.2} \\
TopKSigLIP (k=30) & \textbf{74.1}    & \textbf{86.2}    & 85.3             & \underline{94.6} & \textbf{81.6}    & --  & \underline{82.0} & 93.9             & --  & \underline{77.6} & 76.9    & 94.1             \\
TopKSigLIP (k=45) & \underline{73.6} & 85.9             & 85.1             & \underline{94.6} & \underline{81.3} & --  & 81.6             & 93.8             & --  & 77.3             & 76.3 & \underline{94.2} \\
\bottomrule
\end{tabular}%
}
\end{table*}
\subsection{Results}

\subsubsection{Accuracy}
Tables~\ref{tab:zeroshot} and~\ref{tab:linearprobe} report AUROC under zero-shot and linear probe settings. At inference, $k$ can be freely adjusted beyond the training value of $k{=}15$, covering roughly 10\%, 20\%, and 30\% of the image for $k{=}15, 30, 45$.

\paragraph{Zero-shot.} TopKSigLIP achieves state-of-the-art performance across all tasks and datasets. Gains are largest on the clinically demanding tasks of finding detection and cancer prediction, where fine-grained spatial features are critical, with improvements of $+29.2$ and $+17.8$ AUC points over the prior best on VinDr and EMBED, respectively. Notably, $k{=}15$ covers only 10\% of the image yet achieves competitive performance, with diminishing returns at larger $k$, confirming that mammographic findings are highly localized and sparse patch selection is sufficient. The strong zero-shot performance further suggests that TopKSigLIP can serve as an off-the-shelf model for various clinical prediction tasks and, without any fine-tuning, as a backbone for report generation in future work.

\paragraph{Linear probing.} The performance gap narrows under linear probing, as baselines benefit more from supervised fine-tuning while TopKSigLIP's performance remains largely stable. TopKSigLIP nevertheless retains clear advantages on finding detection and cancer prediction. This suggests that while TopKSigLIP is superior in image--text alignment, its image encoder alone does not capture some key visual features in mammograms as well as the image encoders used in baseline models. We conjecture this is due to two reasons: as a proof-of-concept, TopKSigLIP deliberately employs the lightweight ConvNeXt-Tiny encoder (28.6M parameters), which is significantly smaller than competing backbones such as DaViT (87.9M)~\cite{ding2022davitdualattentionvision} and ViT-B (86M)~\cite{dosovitskiy2020image}; and it is trained on a single GPU with batch size 32, versus batch sizes exceeding 100 in several baselines. We expect performance to improve with larger vision encoders and batch size in future work.

\begin{figure}[t]
     \centering
         \includegraphics[width=0.99\textwidth]{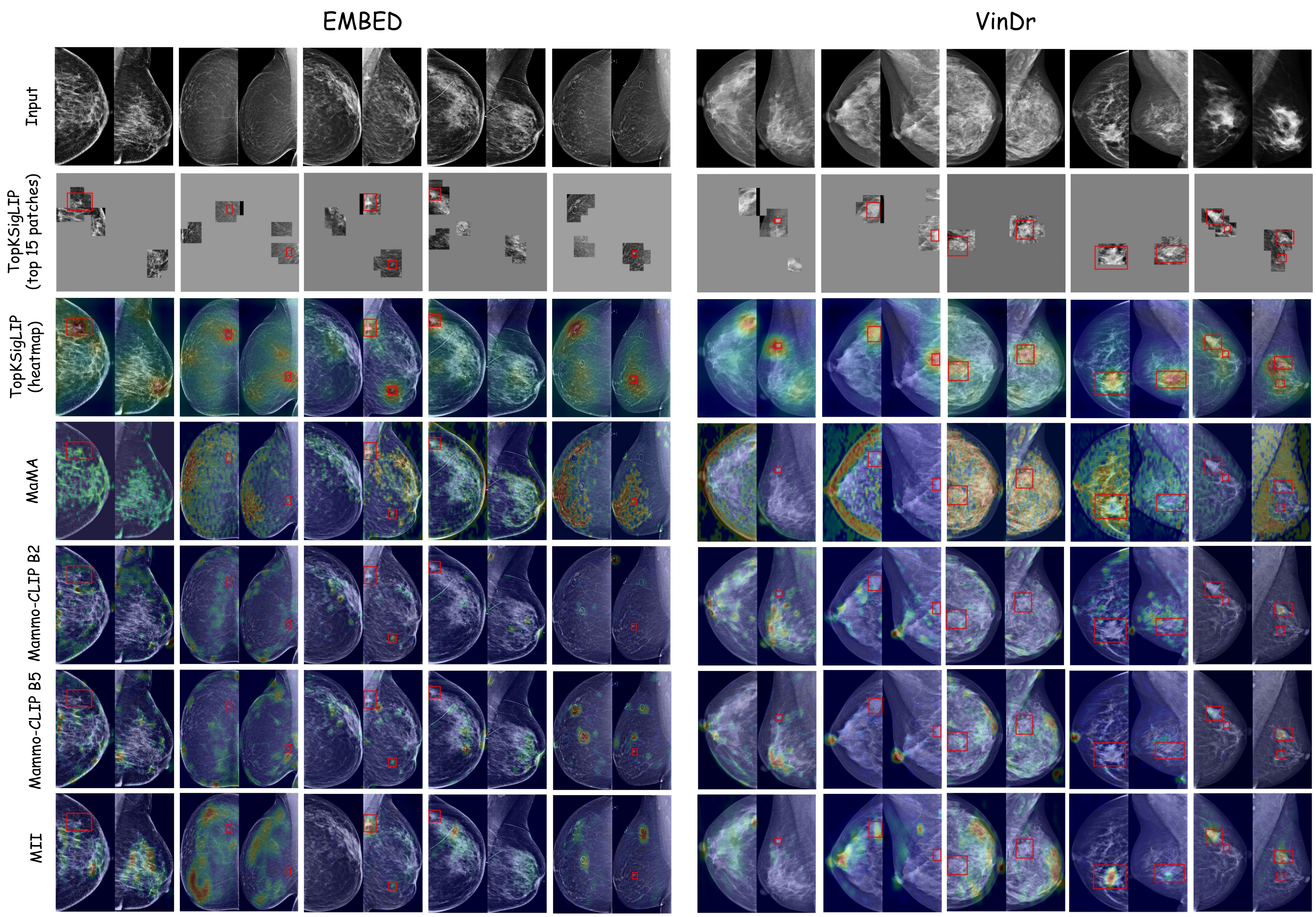}
\caption{\textbf{Localization visualization.} Heatmaps generated by each model on EMBED  and VinDr. Red bounding boxes are ground-truth ROIs. TopKSigLIP's heatmap is the model output $\mathbf{w}$, while all baseline models use \textit{Oracle Grad-CAM} on linear probe in generating the Grad-CAM heatmap. TopKSigLIP consistently produces more accurate heatmap. For TopKSigLIP, the top-15 patches sampled based on the patch importance scores are also shown.}
\label{fig:visualiation}
\end{figure}

\subsubsection{Localization}
\begin{table}{}
\centering
\caption{Localization performance (Average Precision (AP) / Pointing Game (PG)).}
\label{tab:localization}
\definecolor{datasetgray}{gray}{0.90}
\small
\resizebox{0.7\textwidth}{!}{
\begin{tabular}{l|cc|cc}
\toprule
& \multicolumn{2}{c|}{\textbf{Zero-shot}}
& \multicolumn{2}{c}{\textbf{linear probe}} \\
\cmidrule(lr){2-3} \cmidrule(lr){4-5}
Model & EMBED & VinDr & EMBED & VinDr \\
\midrule
MedImageInsight  & 4.2 / 4.0          & 5.4 / 4.2          & \underline{8.0} / \underline{12.7} & \underline{28.8} / \underline{46.2} \\
MammoCLIP-B2     & 5.8 / 10.5        & 14.4 / 38.2         & 4.9 / 11.6          & 20.8 / 37.5         \\
MammoCLIP-B5     & \underline{7.4} /\underline{18.8} & \underline{18.8} / \underline{47.2}  & 5.9 / 11.3 & 21.6 / 34.0 \\
MaMA             & 3.1 / 2.8        & 4.4 / 6.3          & 5.7 / 3.0         & 9.4 / 11.1           \\
TopKSigLIP       & \textbf{14.7} / \textbf{42.1}  & \textbf{28.8} / \textbf{61.8} & --          & --            \\
\bottomrule
\end{tabular}
}
\end{table}

Table~\ref{tab:localization} demonstrates the strong localization performance of TopKSigLIP, whose built-in patch importance map $\mathbf{w}$ consistently outperforms \textit{Oracle Grad-CAM} heatmaps across all baselines in both zero-shot and linear probe settings. Note that the linear probe localization result for TopKSigLIP is omitted, as $\mathbf{w}$ is produced by the frozen vision tower and thus yields identical heatmaps regardless of the linear probe. As shown in Fig.~\ref{fig:visualiation}, while baseline Grad-CAM heatmaps are often visually sharper, they produce many false positives, frequently highlighting high-intensity regions that superficially resemble findings or anatomically salient but clinically irrelevant structures such as the nipple. Among the top-$k{=}15$ patches sampled by TopKSigLIP, false-positive patches tend to correspond to regions with plausible imaging features rather than entirely uninformative areas. We also observe that MaMA attends broadly across breast tissue rather than focusing on specific findings; we conjecture that this global conditioning allows the model to exploit shortcut features, which may partly explain its competitive linear probe performance in Table~\ref{tab:linearprobe} despite poor localization. Localization result of baselines in linear probe setting with Grad-CAM and other interpretation methods such as NormGrad~\cite{rebuffi2019normgradfindingpixelsmatter} without label supervision is provided in Appendix.~\ref{sec:a_grad_cam_on_linear}. We provide more qualitative results of TopKSigLIP in Appendix Fig.~\ref{fig:a_vis_embed}~\ref{fig:a_vis_vindr}.

\subsection{Ablation studies}
\label{sec:ablation}

\begin{table}[hbt]
\centering
\caption{TopK-Patch ablation on EMBED. We compare standard full-image encoding at four resolutions without TopK-Patch against TopKSigLIP across efficiency, prediction accuracy (AUC), and localization (AP). Efficiency is measured with batch size = 32 and max memory 96GB on fp32.} 
\label{tab:ablation_topkpatch}
\resizebox{0.99\textwidth}{!}{%
\begin{tabular}{l|cc|cccc|c}
\toprule
& \multicolumn{2}{c|}{\textbf{Efficiency}}
& \multicolumn{4}{c|}{\textbf{Accuracy}}
& \textbf{Localization} \\
\cmidrule(lr){2-3} \cmidrule(lr){4-7} \cmidrule(lr){8-8}
& Peak Mem (GB)
& MACs (G)
& Findings
& Cancer
& BI-RADS
& Density
& {} \\
\midrule
$256 \times 256$    & 8.94 & 398 & 64.3 & 81.3 & 79.2 & 93.4 & 8.0 \\
$512 \times 512$    & 22.6 & 956 & 68.2 & 82.7 & 80.3 & 92.6 & 7.6 \\
$1024 \times 1024$   & 77.1 & 3191 & 69.7 & 82.9 & \textbf{82.6} & \underline{94.1} & 6.1 \\
$1536 \times 1536$   & OOM & 6915 & \underline{71.5} & \textbf{86.5} & \underline{82.3} & 93.1 &  \underline{12.8} \\
\hline
\noalign{\vspace{4pt}}
w/ ROI sup TopKSigLIP (k=15)  & 64.2 & 2441 & \textbf{73.1} & \underline{85.0} & 80.2 & \textbf{94.6} & \textbf{14.7} \\
w/o ROI sup TopKSigLIP (k=15)  & - & - & 68.8 & 78.8 & 71.9 & 94.5 & 8.9 \\
\bottomrule
\end{tabular}%
}
\end{table}

\paragraph{TopK-Patch.}
Our central claim is that high-res inputs are necessary for accurately detecting mammographic findings, which are often small and easily destroyed by downsampling, yet naively training at full resolution incurs prohibitive GPU memory costs and reduces training batch size. To validate this, we ablate the TopK-Patch by replacing the proposed vision tower with a standard vision tower that encodes the full image at four resolutions: $256^2$, $512^2$, $1024^2$, and $1536^2$, and compare against TopKSigLIP along three axes: efficiency, prediction accuracy, and localization performance with zero-shot \textit{Oracle Grad-CAM}. As shown in Table~\ref{tab:ablation_topkpatch}, prediction accuracy and localization improve with resolution. However, $1024^2$ already exceeds 77 GB and $1536^2$ is OOM at batch size 32, making full-resolution training infeasible. The $1536^2$ training is done with multi-gpu. TopKSigLIP exceeds $1536^2$ on most tasks while consuming only 64.2 GB and remaining trainable at batch size 32. The large drop in accuracy when ROI supervision is removed shows that ROI supervision is essential for training the TopK-Patch module, despite being only $\sim$4\% of training samples.

\begin{table}[hbt]
\centering
\small
\caption{Sup-sigmoid ablation on EMBED in zero-shot setting (AUC). ``Sup.'' indicates supervized; ``Obj.'' indicates training objective.}
\label{tab:ablation_supsigmoid}
\resizebox{0.7\textwidth}{!}{%
\begin{tabular}{cc|cccc}
\toprule
Sup. & Obj. & Findings & Cancer & BI-RADS & Density \\
\midrule
            & Contrastive & 65.7 & 71.5 & 63.4 & 92.8 \\
\checkmark  & Contrastive & 64.7 & 70.1 & 74.3 & 93.1 \\
            & Sigmoid     & 63.6 & 69.5 & 54.1 & 92.1 \\
\checkmark  & Sigmoid     & \textbf{73.1} & \textbf{85.0} & \textbf{80.2} & \textbf{94.6} \\
\bottomrule
\end{tabular}%
}
\end{table}

\paragraph{Sup-sigmoid loss.}
Another claim is that the sigmoid loss is a better learning objective than the contrastive loss when the target matrix $\mathbf{S}$ is augmented by tabular data to be non-identity. Table~\ref{tab:ablation_supsigmoid} compares four combinations of learning objective (contrastive vs.\ sigmoid) and supervision type (identity vs.\ tabular labels), with all other components held fixed; the last row is TopKSigLIP ($k{=}15$). The contrastive loss with supervision, which is identical to MedCLIP's~\cite{wang2022medclip} learning objective, shows only a small benefit, whereas the sigmoid loss yields large gains, confirming that the sigmoid formulation---which treats each pair independently rather than normalizing---is necessary to exploit the non-identity target matrix.

\section{Conclusion and Future Work}
TopKSigLIP is a new mammography VLM that addresses two overlooked challenges in applying CLIP-style pretraining to breast imaging: high-res inputs with sparse findings, and the homogeneity of screening-dominant radiology reports. TopKSigLIP outperforms prior approaches across prediction accuracy and localization on both internal and external benchmarks. Our work demonstrates that architectures and training objectives assumed to be sufficiently general can modifications informed by domain-specific failure modes remain

\newpage
\paragraph{Bias in 2D mammography representation}
TopKSigLIP may introduce certain biases in the learned mammographic representations. The first stems from the TopK-Patch module, which selects patches at a single fixed scale. If a lesion is significantly larger than the patch size, the module risks capturing only a fragment of it, potentially missing broader context (though downstream patch aggregation may partially mitigate this). A natural extension is to predict importance scores at multiple scales, each handling a different patch size. While we expect this to yield incremental improvements, we omit it here as this work is intended as a proof of concept rather than an exhaustive system. Another concern on the TopK-Patch module is error propagation: if the TopK-Patch fails to select patches containing the lesion, this mistake propagates through the rest of the pipeline as there is not skip-connection to restore the lost information. Future work could address this either by strengthening the patch selection module or by incorporating global image information into the final embedding.

The sup-sigmoid loss may also introduce its own biases. By providing explicit similarity labels derived from tabular attributes, the loss may encourage the model to exploit only the corresponding keywords in the report while overlooking finer details such as lesion size, location, and morphology. In contrast, the standard CLIP objective, which simply matches each image to its unique caption, provides a subtler signal that encourages a richer feature corroesponce between the image and text features. We plan to investigate this in future work.

\paragraph{Multi-modality and Longitudinal reasoning}
There are still many open challenges in building a robust mammography VLM deployable in real clinical settings. In clinical practice, breast radiologists rely not only on 2D mammography but also on complementary imaging modalities such as 3D digital breast tomosynthesis (DBT), ultrasound, and MRI, each providing distinct diagnostic information. Furthermore, a radiologist's decision is often informed not by a single exam but by progression over time --- reasoning through a patient's screening-to-diagnostic journey and integrating longitudinal imaging and report data.

The prevailing approach to multi-modal temporal reasoning is to convert all modalities into a unified token sequence and train a next-token prediction model. In principle, such a model could handle diverse tasks --- report generation, image synthesis, and even predicting future disease progression through iterative autoregressive generation. However, realizing this vision for mammography is far from straightforward. As we have emphasized throughout this paper, mammographic images are substantially larger than the natural images typical in computer vision, and the problem is compounded with volumetric modalities: for instance, following the standard $16 \times 16$ convention~\cite{dosovitskiy2020image}, a single DBT volume of size $1024 \times 1024 \times 128$ tokenized at $16 \times 16 \times 16$ yields 32,768 tokens, already exceeding the context length of most modern LLMs during training.

A common strategy to reduce the image token size in LLM-based multi-modal model is a two-stage pipeline: first projecting each image into a compact set of tokens via self-supervised learning or VLM pretraining, then consuming these tokens in a downstream LLM model. For instance, LLaVA~\cite{liu2023visualinstructiontuning,liu2024improved} uses a CLIP-pretrained image encoder to compress an image into a small number of tokens. However, this is only possible if the pretrained image encoder is robustly trained. In this work, we propose such a pretrained encoder for 2D mammography, but analogous encoders would be needed for each additional modality.

\subsubsection{Broader Impact}
\label{sec:a_broader_impact}

% It is in the nature of the field to strive for simpler and more general solutions --- from modality-specific architectures to modality-agnostic ones, and from task-specific models to task-agnostic ones. Yet on this trajectory, many problems have first been addressed through specialized approaches. These specialized solutions, though eventually superseded by more general successors, are far from useless: they serve as lightweight alternatives when efficiency matters, and they provide the building blocks and motivation toward more unified frameworks. TopKSigLIP is one such contribution --- not a final answer, but a step along the way. While our method may not be the general solution that supersedes prior work in mammography or medical VLMs broadly, it offers a new perspective on medical VLM design. We hope it serves as a reminder to the community that architectures and training objectives once assumed to be sufficiently general can falter in specialized medical domains, and that recognizing these failure modes is a necessary step toward truly robust solutions.

Beyond mammography, the two challenges we identify are not unique to breast imaging. High-res inputs are equally critical in histopathology, where whole-slide images can exceed $100{,}000 \times 100{,}000$ pixels~\cite{bae2023data}. Similarly, report homogeneity is a recurring issue across various screening-dominant modalities, exhibiting a strong skew toward normal findings, producing repetitive reports that violate the diversity assumption underlying contrastive learning. The TopK-Patch module and sup-sigmoid loss proposed here are not inherently tied to mammography and may serve as useful starting points for these adjacent domains.

From a clinical deployment perspective, TopKSigLIP's built-in localization offers a practical advantage. Unlike post-hoc methods such as Grad-CAM, which provide saliency maps that can be difficult to validate clinically, the discrete patch locations selected by TopK-Patch are transparent and auditable --- a radiologist can directly verify whether the model is attending to the correct lesion. 

\subsection{Label Statistics}
\label{sec:a_label_statistics}
The three datasets---EMBED, VinDr, and RSNA---span diverse patient demographics (EMBED: 41\% African American; VinDr: Vietnamese cohort; RSNA: US and Australian cohorts) and exhibit highly skewed label distributions. Table~\ref{tab:label_stats} summarizes label statistics at the breast level (merging CC and MLO views), so sample counts are approximately twice the reported exam counts. EMBED statistics are reported on cohorts 1, 2, 8, 9, and 10. ROI count denotes the proportion of images containing at least one ROI annotation (multiple ROIs per image are counted as one). ROI size is the fraction of the image area covered by the ROI foreground, normalized by the image resolution ($1536 \times 1536$).

\begin{table}[h]
\centering
\caption{Label statistics of the three datasets.}
\label{tab:label_stats}
\resizebox{0.8\linewidth}{!}{%
\begin{tabular}{lrrr}
\toprule
& \textbf{EMBED} (Cohort 1,2,8,9,10) & \textbf{RSNA} & \textbf{VinDr} \\
& $(n=330{,}449)$ & $(n=23{,}826)$ & $(n=10{,}041)$ \\
\midrule
\multicolumn{4}{l}{\textit{Findings, n (\%)}} \\
\quad Mass               & 14{,}584 (4.4) & N/A & 618 (6.2)  \\
\quad Calcification      & 16{,}694 (5.1) & N/A & 235 (2.3)  \\
\quad Arch.\ Distortion  & 1{,}861 (0.6)  & N/A & 243 (2.4)  \\
\quad Asymmetry          & 1{,}861 (0.6)  & N/A & 243 (2.4)  \\
\midrule
\multicolumn{4}{l}{\textit{Cancer (positive), n (\%)}} \\
\quad Positive           & 2{,}346 (0.7)  & 492 (2.1)     & N/A        \\
\midrule
\multicolumn{4}{l}{\textit{BI-RADS, n (\%)}} \\
\quad Missing            & 0 (0.0)        & 13{,}221 (55.5) & 0 (0.0)  \\
\quad 0                  & 13{,}695 (4.1) & 3{,}475 (14.6)  & N/A      \\
\quad 1                  & 252{,}764 (76.5) & 6{,}280 (26.4) & 6{,}711 (66.8) \\
\quad 2                  & 45{,}951 (13.9)  & 850 (3.6)      & 2{,}359 (23.5) \\
\quad 3                  & 12{,}557 (3.8)   & N/A            & 467 (4.7)  \\
\quad 4                  & 5{,}190 (1.6)    & N/A            & 387 (3.9)  \\
\quad 5                  & 292 (0.1)        & N/A            & 117 (1.2)  \\
\midrule
\multicolumn{4}{l}{\textit{Density, n (\%)}} \\
\quad Missing            & 0 (0.0)          & 12{,}208 (51.2) & 0 (0.0) \\
\quad A                  & 33{,}540 (10.1)  & 1{,}106 (4.6)   & 51 (0.5)   \\
\quad B                  & 138{,}334 (41.9) & 5{,}022 (21.1)  & 960 (9.6)  \\
\quad C                  & 140{,}455 (42.5) & 4{,}872 (20.4)  & 7{,}676 (76.4) \\
\quad D                  & 18{,}120 (5.5)   & 618 (2.6)       & 1{,}354 (13.5) \\
\midrule
\multicolumn{4}{l}{\textit{ROI}} \\
\quad Count, mean   & 0.045 & 0.000 & 0.097 \\
\quad Size, mean (SD)    & 0.108 (0.116) & N/A           & 0.132 (0.179) \\
\bottomrule
\end{tabular}}
\end{table}

\subsection{Report Parsing}
\label{sec:a_report_parsing}
We use Llama-3~\cite{grattafiori2024llama3herdmodels} to extract clinically relevant content from raw radiology reports. The prompt is provided in Table~\ref{tab:prompt_report_parsing}. An example of the input report and the expected parsed report used for fewshot prompting is provided in Table~\ref{tab:prompt_example}.  Below we describe the key design decisions in the report parsing pipeline.

\paragraph{Exclusion of patient metadata.}
We exclude patient-specific metadata such as age, race, and medical history. Although some of this information may be weakly correlated with imaging features, aligning image representations to demographic attributes brings little clinical benefit and risks encoding spurious correlations.

\paragraph{Exclusion of non-mammographic modalities.}
Findings derived from ultrasound and MRI are removed, as this information cannot be inferred from mammogram images alone. However, we deliberately retain findings and impressions that reference digital breast tomosynthesis (DBT), even though our model processes only 2D mammograms. Our reasoning is that lesions described under DBT findings may still leave subtle traces in the corresponding 2D projections, and exposing the model to these descriptions during training may encourage it to capture signals that are present but difficult to perceive.

\paragraph{Laterality-specific parsing.}
Since the vision tower processes paired CC and MLO views from a single breast, we parse the report into left and right findings separately. Sentences that explicitly specify laterality are assigned exclusively to the corresponding side; sentences referring to both breasts or lacking laterality information are duplicated into both fields.

\paragraph{Enforcing non-empty findings.}
We do not permit empty finding or impression fields for either breast. When no mammographic findings are present, we insert the default sentence: ``No suspicious masses, calcifications or other abnormalities are seen.'' This design choice ensures that the model explicitly learns to represent negative semantics.

\paragraph{BI-RADS score handling.}
The BI-RADS assessment is not extracted from the report text. In standard clinical practice, the BI-RADS score reported in the radiology report reflects the most concerning finding across both breasts. Naively assigning this exam-level score to each breast independently would therefore be incorrect. Instead, we read the breast-level BI-RADS score from the structured tabular data accompanying each exam and convert it into a natural language sentence using the following mapping:
\begin{itemize}
    \item BI-RADS 1: ``N -- Negative.''
    \item BI-RADS 2: ``B -- Benign.''
    \item BI-RADS 3: ``P -- Probably benign.''
    \item BI-RADS 4: ``S -- Suspicious.''
    \item BI-RADS 5: ``M -- Highly suggestive of malignancy.''
\end{itemize}
Note that BI-RADS categories 0 and 6 are excluded. BI-RADS 0 is assigned during screening to flag cases requiring further diagnostic workup. After the diagnostic examination, the patient is reassigned to BI-RADS 1, 2, or 3 if the finding is benign or a screening artifact, or to BI-RADS 4 or 5 if cancer is suspected and biopsy is warranted. In this sense, BI-RADS 0 is not a true diagnostic category but a temporary placeholder prior to a definitive assessment. Although our model is not explicitly trained to predict BI-RADS 0, it can be trivially recovered in a screening setting: if the model's prediction falls within BI-RADS 3--5, this indicates a suspicious finding that would warrant recall, and can be mapped to BI-RADS 0 accordingly. BI-RADS 6 is excluded because it is assigned only after a confirmed tissue diagnosis and thus cannot be inferred from imaging features alone.

\paragraph{Cancer outcome handling.}
Biopsy results are not included in the radiology report itself. Similar to the BI-RADS score handling described above, we read the biopsy-proven cancer outcome from the structured tabular data and append it to the parsed report as a natural language sentence. The original pathology label in EMBED is a multi-class variable reflecting the most severe pathology result from a given specimen, with categories: 0 (invasive cancer), 1 (non-invasive cancer), 2 (high-risk lesion), 3 (borderline lesion), 4 (benign findings), 5 (negative/normal breast tissue), and 6 (non-breast cancer). We binarize this label by grouping categories 0 and 1 as cancer-positive and the remainder as cancer-negative, and convert it into one of the following sentences:
\begin{itemize}
    \item Cancer-negative: ``No evidence of cancer from neither screening nor biopsy.''
    \item Cancer-positive: ``Biopsy confirms the presence of cancer.''
\end{itemize}

\begin{table}[t!]\centering
\caption{LLM prompt for extracting mammography-specific findings from radiology reports. The prompt instructs the model to parse findings by laterality and breast density while discarding non-mammographic information.}
\begin{minipage}{1.0\linewidth}\vspace{0mm}    \centering
\begin{tcolorbox} 
    \centering
      \footnotesize
\begin{tabular}{p{0.97\linewidth} c}
{\bf messages} = [``role'':``{\bf system}'', ``content'': \\
f```You are an expert breast radiologist. Given a breast radiology report, extract only the key findings that are inferable from mammography (FFDM) and/or digital breast tomosynthesis (DBT). Exclude any information derived from ultrasound (US) or MRI. Your output MUST be a JSON dictionary with EXACTLY three fields: ``left'', ``right'', ``density''. \\
\\
\underline{General Rules}:\\
\emph{(1)} COPY the text EXACTLY as written --- DO NOT rephrase, summarize, or restructure sentences.\\
\emph{(2)} Include ONLY information inferable from mammography (FFDM/DBT) --- Discard any information derived from US, MRI, clinical history, or demographics.\\
\emph{(3)} You MUST NOT leave any field empty. Values in each field must be a list of sentences.\\
\\
\underline{Key-Specific Rules}:\\
\emph{left / right:}\\
\emph{(1)} DO NOT include densities.\\
\emph{(2)} Include only findings inferable from mammography (FFDM/DBT): masses, calcifications, asymmetry, architectural distortion, or explicitly stated absence of mammographic findings.\\
\emph{(3)} In the IMPRESSION section, EXCLUDE any recommendations or follow-up advice.\\
\emph{(4)} If a sentence specifies laterality (left/right), assign it exclusively to that field. If it refers to both breasts or has no laterality, copy the EXACT SAME sentence into both ``left'' and ``right''.\\
\emph{(5)} If NO mammographic findings are present, insert EXACTLY: ``No suspicious masses, calcifications or other abnormalities are seen'' as the only item in both fields.\\
\\
\emph{density:}\\
\emph{(1)} Extract from the MAMMOGRAM FINDINGS section. Must be EXACTLY one of: ``almost entirely fat'', ``Scattered fibroglandular densities'', ``Heterogeneously dense'', ``Extremely dense''.\\
\\
\underline{Here are some cases}:
(1)... (2)... (3)... (4)... \\
Organize your output as a JSON-formatted dictionary: dict[str, list[str]] with EXACTLY the three fields above, without other words.''']\\
\hrulefill & \\
{\bf messages} += [``role'':``{\bf user}'', ``content'': ``Input: \{report\}''\\
\end{tabular}
\end{tcolorbox}
\label{tab:prompt_report_parsing}
\end{minipage}
\end{table}

\begin{table}[t!]\centering
\caption{Example input--output pair for the report parsing prompt (Table~\ref{tab:prompt_report_parsing}). The LLM extracts laterality-specific findings, breast density, and BI-RADS category while discarding clinical history, recommendations, and non-mammographic information.}
\begin{minipage}{1.0\linewidth}\vspace{0mm}    \centering
\begin{tcolorbox} 
    \centering
      \footnotesize
\begin{tabular}{p{0.97\linewidth}}
\underline{\bf Input Report}:\\[2pt]
HISTORY: \\
Patient is xx years old and is seen for screening. The patient has a history of bilateral breast reduction in xxxx. \\[2pt]
FILMS COMPARED:\\ No prior imaging studies are available for comparison. \\[2pt]
MAMMOGRAM FINDINGS: \\
The following mammographic views were obtained: bilateral craniocaudal, bilateral mediolateral oblique, and bilateral tomosynthesis. Computer-aided detection was utilized by the radiologist in the interpretation of this examination. \\[2pt]
The breasts are heterogeneously dense. This may lower the sensitivity of mammography. \\[2pt]
Finding 1: There are areas of post-reduction change seen in both breasts. \\
Finding 2: There is an isodense focal asymmetry measuring 15 millimeters with spiculated margins seen in the middle region of the right breast upper outer quadrant at 10 o'clock. This is best visualized on tomosynthesis CC view slice \# 30 and MLO view slice \# 21. \\
Finding 3: There is an isodense focal asymmetry measuring 10 millimeters with spiculated margins seen in the posterior region of the right breast at 9 o'clock. \\
Finding 4: There is an isodense focal asymmetry measuring 15 millimeters with spiculated margins seen in the anterior region of the left breast upper outer quadrant at 2 o'clock. \\[2pt]
IMPRESSION: \\
Finding 1: Areas of post-reduction change in both breasts appear benign. \\
Finding 2: Focal asymmetry in the middle region of the right breast at 10 o'clock requires additional evaluation. Additional mammographic images are recommended. A possible ultrasound may be warranted following the mammographic views. \\
Finding 3: Focal asymmetry in the posterior region of the right breast at 9 o'clock requires additional evaluation. Additional mammographic images are recommended. A possible ultrasound may be warranted following the mammographic views. \\
Finding 4: Focal asymmetry in the left breast requires additional evaluation. Additional mammographic images are recommended. A possible ultrasound may be warranted following the mammographic views. \\[2pt]
BI-RADS Category 0: Additional evaluation needed. \\[4pt]
\hrulefill \\[4pt]
\underline{\bf Extracted Output}:\\[2pt]
``left'': [ \\
\quad ``There are areas of post-reduction change seen in both breasts.'', \\
\quad ``There is an isodense focal asymmetry measuring 15 millimeters with spiculated margins seen in the anterior region of the left breast upper outer quadrant at 2 o'clock.'', \\
\quad ``Areas of post-reduction change in both breasts appear benign.'', \\
\quad ``Focal asymmetry in the left breast requires additional evaluation.''] \\[2pt]
``right'': [ \\
\quad ``There are areas of post-reduction change seen in both breasts.'', \\
\quad ``There is an isodense focal asymmetry measuring 15 millimeters with spiculated margins seen in the middle region of the right breast upper outer quadrant at 10 o'clock.'', \\
\quad ``Areas of post-reduction change in both breasts appear benign.'', \\
\quad ``Focal asymmetry in the middle region of the right breast at 10 o'clock requires additional evaluation.'', \\
\quad ``Focal asymmetry in the posterior region of the right breast at 9 o'clock requires additional evaluation.''] \\[2pt]
``density'': [``The breasts are heterogeneously dense. This may lower the sensitivity of mammography.''] \\[2pt]
\end{tabular}
\end{tcolorbox}
\label{tab:prompt_example}
\end{minipage}
\end{table}

\begin{table}[h!]
\centering
\footnotesize
\caption{Zero-shot prompts for \textbf{density} classification.}
\label{tab:zs_density}
\begin{tabular}{lp{10cm}}
\toprule
\textbf{Model} & \textbf{Prompts (Density 1 - 4) } \\
\midrule
TopKSigLIP/MII & ``The breast is almost entirely fatty.'', ``The breast is scattered with fibroglandular densities.'',``The breast is heterogeneously dense.'', ``The breast is extremely dense.''
 \\
\midrule
MaMA/GLAM & ``Breast composition: The breast is almost entirely fat.'', ``Breast composition: The breast is scattered fibroglandular densities. '', ``Breast composition: The breast is heterogeneously dense. This may lower the sensitivity of mammography. '', ``Breast composition: The breast is extremely dense. This may lower the sensitivity of mammography. ''
 \\ 
\midrule
Mammo-CLIP B2/B5 & ``the breasts being almost entirely fatty'', ``scattered areas of fibroglandular density'', ``the breast tissue is heterogeneously dense'', ``the breasts are extremely dense'' \\
\bottomrule
\end{tabular}
\end{table}

\begin{table}[h!]
\centering
\footnotesize
\caption{Zero-shot prompts for \textbf{BI-RADS} classification.}
\label{tab:zs_birads}
\begin{tabular}{lp{10cm}}
\toprule
\textbf{Model} & \textbf{Prompts (BIRADS 1-5)} \\
\midrule
TopKSigLIP/MII &  ``BIRADS 1: N – Negative'', ``BIRADS 2: B - Benign'', ``BIRADS 3: P – Probably benign'', ``BIRADS 4: S – Suspicious'', ``BIRADS 5: M - Highly suggestive of malignancy'' \\
\midrule
MaMA/GLAM &  ``Findings: The mammogram shows that no significant masses, calcification, or other abnormalities are present. '', ``Findings: The mammogram shows that a benign finding is present. '', ``Findings: The mammogram shows that a probably benign finding is present. '', ``Findings: The mammogram shows that a suspicious abnormality is present. '', ``Findings: The mammogram shows that a highly suggestive of malignancy is present, a biopsy is recommended. '' \\ 
\midrule
Mammo-CLIP B2/B5 &  ``birads category 1'', ``birads category 2'', ``birads category 3'', ``birads category 4'', ``birads category 5'' \\
\bottomrule
\end{tabular}
\end{table}

\begin{table}[h!]
\centering
\footnotesize
\caption{Zero-shot prompts for \textbf{cancer outcome} prediction.}
\label{tab:zs_cancer}
\begin{tabular}{lp{10cm}}
\toprule
\textbf{Model} & \textbf{Prompts (negative / positive)} \\
\midrule
TopKSigLIP & ``No evidence of cancer from neither screening nor biopsy.'' / ``Biopsy confirms the presence of cancer.''
 \\
\midrule
MII & ``normal'' /``cancer'' \\
\midrule
MaMA / GLAM & ``Findings: The mammogram shows that Cancer negative: overall healthy or just benign finding'' / ``Findings: The mammogram shows that Cancer positive: screening image with known biopsy-proven malignancy or suspicious abnormality found'' \\
\midrule
Mammo-CLIP B2/B5 & ``no malignancy'' / ``malignancy'' \\
\bottomrule
\end{tabular}
\end{table}

\begin{table}[h!]
\centering
\footnotesize
\caption{Zero-shot prompts for \textbf{finding type}. Each finding is evaluated as an independent binary classification with a [negative, positive] prompt pair.}
\label{tab:zs_finding}
\begin{tabular}{llp{9cm}}
\toprule
\textbf{Model} & \textbf{Finding} & \textbf{Prompts [negative, positive]} \\
\midrule
\multirow{4}{*}{TopKSigLIP} 
& MASS & [``No suspicious masses, calcifications or other abnormalities are seen.'', ``there is a mass seen in the breast.''] \\
& CALC & [``No suspicious masses, calcifications or other abnormalities are seen.'', ``there are calcifications seen in the breast.''] \\
& AD & [``No suspicious masses, calcifications or other abnormalities are seen.'', ``there is an architectural distortion seen in the breast.''] \\
& AS & [``No suspicious masses, calcifications or other abnormalities are seen.'', ``there is an asymmetry seen in the breast.''] \\
\midrule
\multirow{4}{*}{MII/ MaMA / GLAM} 
& MASS & [``no suspicious mass'', ``suspicious mass''] \\
& CALC & [``no suspicious calcifications'', ``suspicious calcification''] \\
& AD & [``no suspicious architectural distortions'', ``architectural distortion''] \\
& AS & [``no suspicious asymmetries'', ``asymmetry''] \\
\midrule
\multirow{4}{*}{Mammo-CLIP B2/B5} 
& MASS & [``no mass'', ``mass''] \\
& CALC & [``no suspicious calcification'', ``suspicious calcification''] \\
& AD & [``no suspicious architectural distortions'', ``architectural distortion''] \\
& AS & [``no suspicious asymmetries'', ``asymmetry''] \\
\bottomrule
\end{tabular}
\end{table}

\subsection{Evaluation protocol}
\label{sec:a_evaluation_protocol}

\subsubsection{Zero-shot evaluation}
\label{sec:a_evaluation_protocol_zero_shot}
For zero-shot evaluation, we adopt each model's published prompt template where available. For models that do not provide prompts for certain tasks, we design prompts that best align with the model's training vocabulary. Tables~\ref{tab:zs_density}--\ref{tab:zs_finding} summarize the prompts used for each model across all four tasks.

\begin{figure}[t]
     \centering
         \includegraphics[width=0.99\textwidth]{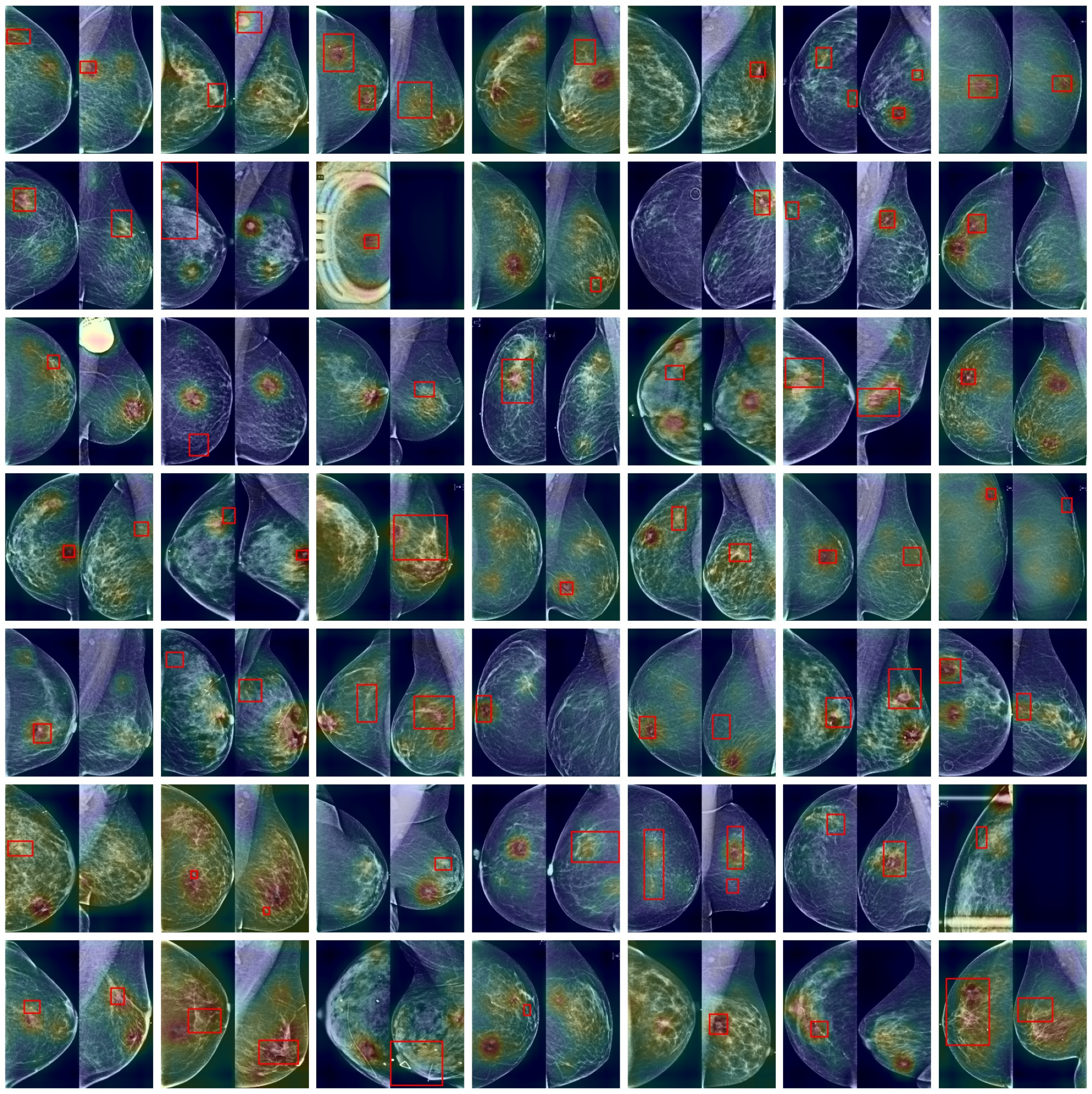}
\caption{\textbf{Localization visualization on random samples from EMBED.}}
\label{fig:a_vis_embed}
\end{figure}

\begin{figure}[t]
     \centering
         \includegraphics[width=0.99\textwidth]{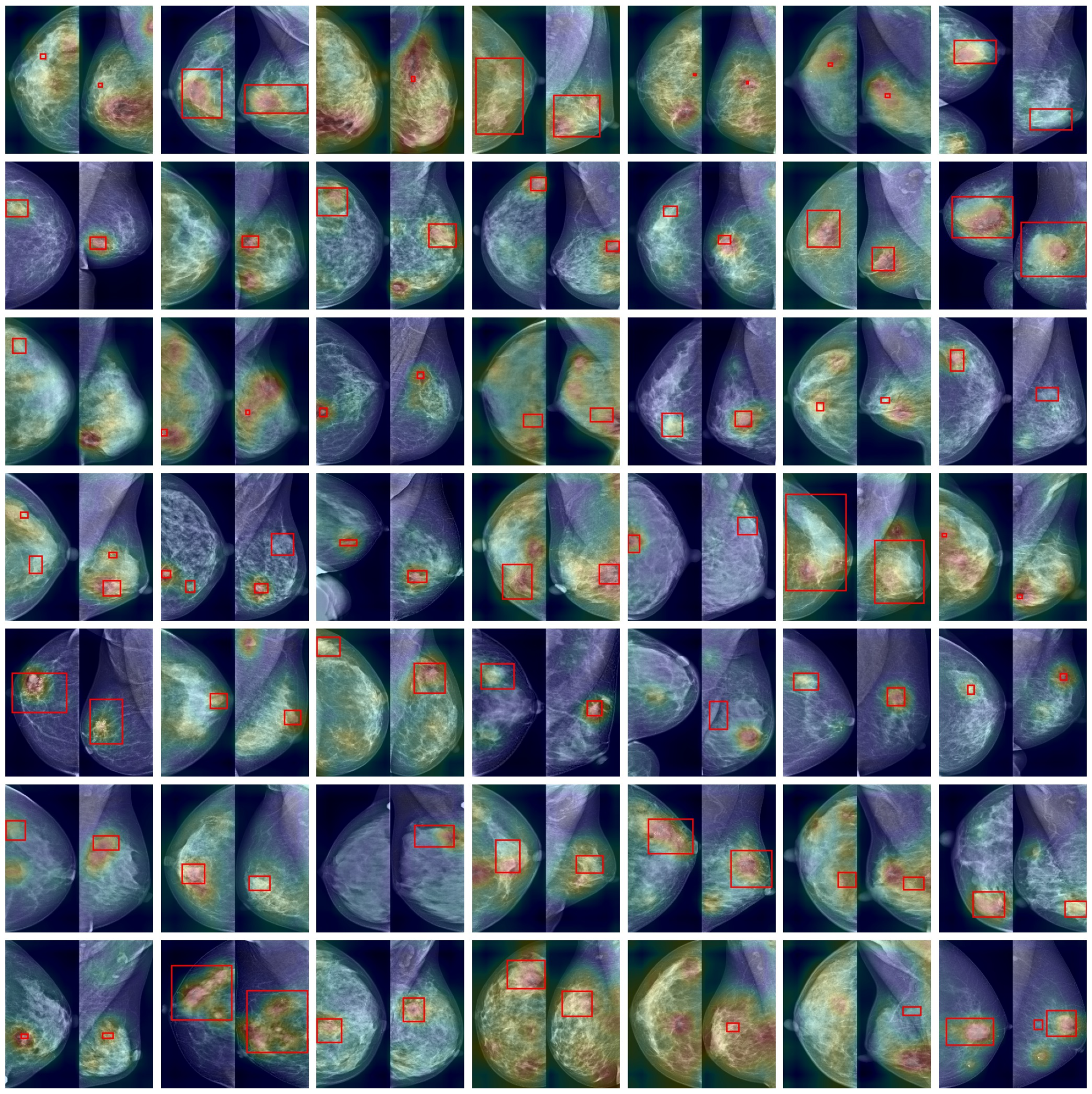}
\caption{\textbf{Localization visualization on random samples from VinDr.}}
\label{fig:a_vis_vindr}
\end{figure}

\subsubsection{Localization evaluation}
\label{sec:a_localization_evaluation}

As described in the main text, we use \textit{Oracle Grad-CAM} rather than standard Grad-CAM for the baseline models. Standard Grad-CAM uses the gradient of the model's own predicted logit; however, baseline models performed too poorly under this setting to yield meaningful heatmaps. We therefore provide an advantage to the baselines by supplying the ground-truth finding type, cancer outcome, and BI-RADS score when generating the final logit.

For zero-shot \textit{Oracle Grad-CAM}, each ground-truth label is converted into a natural language prompt following the zero-shot prompt protocol described in Sec.~\ref{sec:a_evaluation_protocol_zero_shot}. For instance, for MaMA, a sample with a mass finding, confirmed cancer, and BI-RADS 5 assessment would yield the prompt: \texttt{``suspicious mass. Findings: The mammogram shows that a finding highly suggestive of malignancy is present, a biopsy is recommended. Findings: The mammogram shows that cancer positive: screening image with known biopsy-proven malignancy or suspicious abnormality found.''} The ground-truth logit is then computed as the inner product between the image embedding and the text embedding of the constructed prompt. Gradients of this logit with respect to the image encoder's feature map prior to spatial pooling are used to produce the final heatmap.

For linear probe \textit{Oracle Grad-CAM}, we use the logits from the finding, cancer, and BI-RADS heads. For the finding head (multi-label), the ground-truth binary label vector masks the logits so that only positive findings contribute, and the masked logits are summed across classes. For the cancer and BI-RADS heads (multi-class), the ground-truth class index directly selects the corresponding logit. The three resulting per-sample scalars are summed to form a single aggregated logit.

\begin{table}
\caption{Localization performance (AP/PG) with Grad-CAM and NormGrad (no label supervision) on linear probe model. \textit{Oracle Grad-CAM} results are also shown for reference.}
\label{tab:a_grad-cam}
\definecolor{datasetgray}{gray}{0.90}
\small
\resizebox{\textwidth}{!}{
\begin{tabular}{l|cccc c|cccc c}
\toprule
& \multicolumn{5}{c|}{\textbf{EMBED}} & \multicolumn{5}{c}{\textbf{VinDr}} \\
\cmidrule(lr){2-6} \cmidrule(lr){7-11}
& BI-RADS & Cancer & Finding & All & All (Oracle) & BI-RADS & Cancer & Finding & All & All (Oracle) \\
\midrule
\rowcolor{gray!20} & \multicolumn{10}{c}{\textbf{Grad-CAM}} \\
\midrule
MedImageInsight  & 3.8 / 8.5 & 3.3 / 9.1 & 7.4 / 15.8 & 9.6 / 14.1 & 8.0 / 12.7 & 13.8 / 28.5 & 4.7 / 7.6 & 24.1 / 44.1 & 25.3 / 41.7 & 28.8 / 46.2 \\
MammoCLIP-B2     & 3.7 / 8.7 & 3.7 / 2.3 & 4.9 / 12.9 & 6.9 / 11.6 & 4.9 / 11.6 & 10.5 / 18.8 & 6.3 / 7.6 & 16.1 / 31.3 & 16.8 / 27.1 & 20.8 / 37.5 \\
MammoCLIP-B5     & 3.4 / 10.4 & 3.2 / 20.8 & 4.5 / 13.7 & 6.6 / 11.9 & 5.9 / 11.3 & 10.6 /19.4 & 16.8 / 39.6 & 16.5 / 29.2 & 17.6 / 26.4 & 21.6 / 34.0 \\
MaMA             & 3.1 / 4.0 & 4.3 / 2.9 & 4.3 / 3.6 & 4.8 / 3.4 & 5.7 / 3.0 & 6.3 /7.6  & 8.2 / 11.5 & 8.2 / 9.4 & 8.3 / 9.7 & 9.4 / 11.1 \\
\midrule
\rowcolor{gray!20} & \multicolumn{10}{c}{\textbf{NormGrad}} \\

\midrule
MedImageInsight & 6.1 / 8.4 & 6.5 / 7.3 & 6.3 / 8.6 & 5.9 / 8.6 & 5.9 / 8.2 & 9.4 / 12.5 & 10.2 / 8.0 & 13.5 / 13.5 & 14.2 / 12.5 & 15.6 / 14.6 \\
MammoCLIP-B2     & 6.6 / 6.2 & 6.6 / 6.4 & 6.4 / 5.7 & 5.9 / 5.4 & 6.2 / 5.0 & 16.2 / 15.6 & 14.1 / 13.5 & 19.2 / 11.5 & 19.9 / 12.8 & 19.9 / 13.9 \\
MammoCLIP-B5     & 7.1 / 7.4 & 7.2 / 6.2 & 6.7 / 4.8 & 6.6 / 5.3 & 6.4 / 5.2 & 16.5 / 12.2 & 13.8 / 10.1 & 19.4 / 8.7 & 19.7 / 6.6 & 19.4 / 10.4 \\
MaMA             & 3.3 / 4.1 & 2.8 / 2.4 & 3.1 / 2.6 & 3.4 / 2.9 & 3.5 / 3.0 & 4.5 / 7.3 & 4.3 / 6.9 & 5.3 / 5.9 & 7.2 / 5.2 & 7.7 / 5.6 \\
\bottomrule
\end{tabular}
}
\end{table}

\subsection{Grad-CAM and NormGrad on linear probe}
\label{sec:a_grad_cam_on_linear}
Tab.~\ref{tab:a_grad-cam} shows the localization results of baselines in linear probe setting without relying on the ground truth. The gradients are thus computed on the highest logit value. Also, we vary the combination of logit heads in generating the heatmap. 

\subsection{Proof: Differentiable Top-$k$ as Exact Sampling from the Plackett--Luce Distribution}
\label{sec:proof_top_k}
Given importance scores $\mathbf{w} \in \mathbb{R}_{\geq 0}^{N}$, sampling $K$ items follows the Plackett--Luce distribution~\cite{plackett1975analysis}:
\begin{equation}
    p(E \mid \mathbf{w}) = 
    \frac{w_{k_1}}{Z} \cdot \frac{w_{k_2}}{Z - w_{k_1}} \cdots 
    \frac{w_{k_K}}{Z - \sum_{j=1}^{K-1} w_{k_j}},
\end{equation}
where $Z = \sum_{i=1}^{N} w_i$. TopK-Patch adopts the Gumbel-top-$K$ reparameterization~\cite{xie2019reparameterizable}, which perturbs log-weights with Gumbel noise, $\hat{r}_i = \log(w_i) + g_i$ where $g_i \sim \mathrm{Gumbel}(0,1)$, and applies the iterative softmax relaxation~\cite{plotz2018neural} (Eq.~\ref{eq:softopk}). We show that as $\tau \to 0$, this yields exact samples from $p(E \mid \mathbf{w})$.

The proof proceeds in two parts. First, we show that taking the hard top-$K$ of Gumbel-perturbed log-weights yields exact Plackett--Luce samples, following the connection established by Vieira~\cite{vieira2014gumbel} and Xie and Ermon~\cite{xie2019reparameterizable}. Second, we show that the iterative softmax relaxation recovers this hard top-$K$ as $\tau \to 0$.

\paragraph{Part 1: Gumbel-top-$K$ equivalence to WRS.}
The proof relies on the weighted reservoir sampling (WRS) algorithm of Efraimidis and Spirakis~\cite{efraimidis2006weighted}, which generates random keys $r_i = u_i^{1/w_i}$ where $u_i \sim \mathrm{Uniform}(0,1)$, and returns the indices of the top-$K$ keys. Efraimidis and Spirakis~\cite{efraimidis2006weighted} proved that the output of WRS is distributed exactly according to $p(E \mid \mathbf{w})$.

\begin{proposition}
$\mathrm{TopK}(\hat{\mathbf{r}}, K) = \mathrm{TopK}(\mathbf{r}, K)$ almost surely when using the same source of randomness.
\end{proposition}
\begin{proof}
A $\mathrm{Gumbel}(0,1)$ variate can be generated as $g_i = -\log(-\log(u_i))$ where $u_i \sim \mathrm{Uniform}(0,1)$. Substituting:
\begin{equation}
    \hat{r}_i = \log(w_i) - \log(-\log(u_i)) = -\log\!\left(-\tfrac{1}{w_i}\log(u_i)\right) = -\log(-\log(u_i^{1/w_i})) = -\log(-\log(r_i)).
\end{equation}
Since $x \mapsto -\log(-\log(x))$ is strictly monotonic on $(0,1)$, the transformation preserves ordering: $\hat{r}_i \leq \hat{r}_j \iff r_i \leq r_j$. Therefore $\mathrm{TopK}(\hat{\mathbf{r}}, K) = \mathrm{TopK}(\mathbf{r}, K)$, and taking the top-$K$ elements of $\hat{\mathbf{r}}$ yields exact samples from $p(E \mid \mathbf{w})$.
\end{proof}

\paragraph{Part 2: Iterative softmax recovers hard top-$K$ as $\tau \to 0$.}
At step $j=1$, as $\tau \to 0$ the softmax converges to a hard argmax, so $e_i^{(k_1)} \to \mathbf{1}[i = k_1^*]$ where $k_1^* = \arg\max_i \hat{r}_i$. The update rule then gives:
\begin{equation}
    \alpha_i^{(k_2)} = \hat{r}_i + \log(1 - e_i^{(k_1)}) = 
    \begin{cases}
        -\infty & \text{if } i = k_1^*, \\
        \hat{r}_i & \text{otherwise},
    \end{cases}
\end{equation}
which removes $k_1^*$ from consideration. At step $j=2$, the softmax selects $k_2^* = \arg\max_{i \neq k_1^*} \hat{r}_i$, and the update again sends $\alpha_{k_2^*}$ to $-\infty$. By induction, after $K$ steps the selected indices $\{k_1^*, \ldots, k_K^*\}$ are the $K$ largest elements of $\hat{\mathbf{r}}$, recovering the hard top-$K$ operation. Combined with Part 1, this establishes that as $\tau \to 0$, the iterative softmax relaxation yields exact samples from $p(E \mid \mathbf{w})$.

\end{document}